\documentclass[final]{siamart0516}
\pdfoutput=1 
\usepackage{braket,amsfonts}
\usepackage{array}
\usepackage{pgfplots}
\usepackage[caption=false]{subfig}
\usepackage{tablefootnote}

\newsiamthm{claim}{Claim}
\newsiamremark{rem}{Remark}
\newsiamremark{expl}{Example}
\newsiamremark{hypothesis}{Hypothesis}
\crefname{hypothesis}{Hypothesis}{Hypotheses}

\usepackage{algorithmic}

\usepackage{graphicx,epstopdf}
\usepackage{float}

\Crefname{ALC@unique}{Line}{Lines}

\usepackage{amsopn}

\newcommand{\cN}{\mathcal{N}}

\newcommand{\one}{\mathbf{1}}

\newcommand{\Jk}{\mathsf{J}}
\newcommand{\Jp}{\mathsf{J}_{\rm p}}
\newcommand{\Jl}{\mathsf{J}_{\rm ls}}

\newcommand{\Jg}{\mathsf{J}_{\rm gl}}
\newcommand{\Pp}{\mathbb{P}_{\rm p}}
\newcommand{\Pl}{\mathbb{P}_{\rm ls}}
\newcommand{\Pg}{\mathbb{P}_{\rm gl}}

\newcommand{\Phip}{\Phi_{\rm p}}
\newcommand{\Phil}{\Phi_{\rm ls}}
\newcommand{\Phig}{\Phi_{\rm gl}}

\newcommand{\mup}{\mu_{\rm p}}
\newcommand{\mul}{\mu_{\rm ls}}

\newcommand{\nup}{\nu_{\rm p}}
\newcommand{\nul}{\nu_{\rm ls}}
\newcommand{\nug}{\nu_{\rm gl}}

\newcommand{\nyst}{Nystr\"{o}m }

\definecolor{darkred}{rgb}{0.6,0.1,0.1}
\definecolor{darkgreen}{rgb}{0.1,0.6,0.1}
\definecolor{darkblue}{rgb}{0.1,0.1,0.6}

\newcommand{\GL}{\mathsf{GL}}
\newcommand{\Se}{S_{\epsilon}}
\newcommand{\We}{W_{\epsilon}}

\newcommand{\bbE}{\mathbb{E}}

\newcommand{\bbR}{\mathbb{R}}
\newcommand{\bbP}{\mathbb{P}}
\newcommand{\bbI}{\mathbb{I}}

\newcommand{\be}{\begin{equation}}
\newcommand{\en}{\end{equation}}

\newcommand{\modk}[1]{\textcolor{blue}{#1}\index {#1}}

\title{Uncertainty Quantification in Graph-Based Classification of High Dimensional Data
  \thanks{Submitted to the editors DATE.
\funding{AB is funded by NSF grant DMS-1118971, NSF
grant DMS-1417674, and ONR grant N00014-16-1-2119. AMS is funded by DARPA (EQUIPS) 
and EPSRC (Programme Grant).  KCZ was partially supported by a grant from the
Simons Foundation and by the Alan Turing Institute under the EPSRC grant EP/N510129/1. Part
of this work was done during the authorÕs stay at the Newton Institute for the program Stochastic
Dynamical Systems in Biology: Numerical Methods and Applications. }}}

\author{ Andrea L Bertozzi 
  \thanks{Department of Mathematics, University of California Los Angeles, Los Angeles, CA (\email{bertozzi@math.ucla.edu}).}%
  \and
  Xiyang Luo
  \thanks{Department of Mathematics, University of California Los Angeles, Los Angeles, CA (\email{mathluo@math.ucla.edu}).}
  \and
  Andrew M. Stuart
   \thanks{Computing and Mathematical Sciences, Caltech, Pasadena, CA (\email{astuart@caltech.edu }).}
  \and
  Konstantinos C. Zygalakis
   \thanks{School  of Mathematics, University of Edinburgh, Edinburgh, Scotland 
    (\email{k.zygalakis@ed.ac.uk }).}
}

\begin{document}
\maketitle

\begin{abstract}
Classification of high dimensional data finds wide-ranging 
applications. In many of these applications equipping the resulting
classification with  a measure of uncertainty may be as important
as the classification itself. In this paper we introduce, develop
algorithms for, and
investigate the properties of, a variety of Bayesian models for the 
task of binary classification; via the posterior distribution
on the classification labels, these methods automatically give 
measures of uncertainty. The methods are all based around the graph
formulation of semi-supervised learning. 

We provide a unified framework which brings together
a variety of methods which have been introduced in different
communities within the mathematical sciences. We study probit classification 
\cite{williams1996gaussian} in the graph-based setting, 
generalize the level-set method for
Bayesian inverse problems \cite{iglesias2015bayesian} to
the classification setting, and generalize the Ginzburg-Landau optimization-based classifier \cite{bertozzi2012diffuse,van2012gamma} to a Bayesian setting;
we also show that the probit and level set approaches are natural
relaxations of the harmonic function approach introduced in
\cite{zhu2003semi}.  
We introduce efficient numerical methods, suited to large data-sets,
for both MCMC-based sampling as well as gradient-based MAP estimation. 
Through numerical experiments we study classification accuracy and 
uncertainty quantification for our models; these experiments showcase
a suite of datasets commonly used to evaluate graph-based semi-supervised 
learning algorithms. 
\end{abstract}




\section{Introduction}
\label{sec:intro}
\subsection{The Central Idea}
Semi-supervised learning has attracted the attention of many researchers
because of the importance of combining unlabeled data with labeled data. 
In many applications the number of unlabeled data points is so large 
that it is only feasible to label a small subset of these points by hand. 
Therefore, the problem of effectively utilizing correlation information in 
unlabeled data to automatically extend the labels to all the data points 
is an important research question in machine learning. This paper concerns the 
issue of how to address  uncertainty quantification in such classification methods.
In doing so we bring together a variety of themes from the mathematical
sciences, including optimization, PDEs, probability and statistics.   
We will show that a variety of different methods, arising in very distinct
communities, can all be formulated around a common objective function 
$$\quad \Jk(w)=\frac12 \langle w, Pw \rangle + \Phi(w)$$ 
for a real valued function $w$ on the nodes of a graph representing the
data points. The matrix $P$ is proportional to a graph Laplacian derived
from the unlabeled data and the function $\Phi$ involves the
labelled data. The variable $w$ is 
used for classification. Minimizing this objective function 
is one approach to such a classification. However if uncertainty
is modeled then it is natural to consider the probability distribution 
with density $\bbP(w)$ proportional to $\exp\bigl(- \Jk(w)\bigr);$ 
this will be derived using Bayesian formulations of the problem.

We emphasize that the variety of probabilistic 
models considered in this paper arise from different assumptions
concerning the structure of the data. Our objective is to introduce
a unified framework within which to propose and evaluate algorithms 
to sample $\bbP(w)$ or minimize $\Jk(w).$ 
We will focus on Monte Carlo
Markov Chain (MCMC) and gradient descent. 
The objective of the paper is not to assess the 
validity of the assumptions leading to the different models; this is an 
important modeling question best addressed through understanding of
specific classes of data arising in applications.

\subsection{Literature Review}

There are a large number of possible approaches to the semi-supervised
learning problem, developed in both the statistics and machine learning
communities. The review \cite{zhu2005semi} provides an excellent overview.
In this paper we will concentrate exclusively on graph-based methods.
These have the attractive feature of reducing potentially high dimensional
unlabeled data to a real-valued function on the edges of a graph, 
quantifying affinity between the nodes; each node of the graph is 
identified with an unlabeled data point. 

Early graph-based approaches were combinatorial. 
Blum et al. posed the binary semi-supervised classification problem as a 
Markov random field (MRF) over the discrete state space of
binary labels, the MAP estimation of which can be solved using a graph-cut algorithm
in polynomial time \cite{zhu2005semi} . In general, inference for multi-label discrete MRFs is intractable \cite{dahlhaus1992complexity}. However, several approximate algorithms exist for the multi-label case \cite{boykov2001fast, boykov1998markov, madry2010fast}, and have been applied to many imaging
tasks \cite{boykov2001interactive, berthod1996bayesian, li2012markov}.

A different line of work is based on 
using the affinity function on the edges to define  
a real-valued function $w$ on the nodes of the graph, a substantial
compression of high dimensional unlabelled data at each node. The
Dirichlet energy $\Jk_0(w):=\frac{1}{2}\langle w, Pw\rangle$, with $P$ proportional
to the graph Laplacian formed from the affinities on the edges, plays a central
role. A key conceptual issue in 
the graph-based approach is then to connect the labels, which 
are discrete, to this real-valued function.   
Strategies to link the discrete and continuous data then constitute different modeling
assumptions. The line of work initiated in  \cite{zhu2003semi} makes the assumption that
the labels are also real-valued and take the real values $\pm 1$, linking
them directly to the real-valued function on the nodes of the graph; this
may be thought of as a continuum relaxation of the discrete state space 
MRF in \cite{blum2001learning}. The basic method is to
minimize $\Jk_0(w)$ subject to the hard constraint that $w$ agrees
with the label values; alternatively this constraint may be relaxed to a soft additional
penalty term added to $\Jk_0(w).$  These methods are a form of krigging, or 
Gaussian process regression \cite{wahba1990spline,williams1996gaussian}, on a graph. 
A Bayesian interpretation of the approach in \cite{zhu2003semi} 
is given in \cite{zhu2003semib} with further developments 
in \cite{kapoor2005hyperparameter}. 
The Laplacian based approach has since been generalized in \cite{zhou2004learning, belkin2006manifold, talukdar2009new, subramanya2011semi,lu2011multi}; in particular this line of work developed to study the transductive
problem of assigning predictions to data points off the graph. 
A formal framework for graph-based regularization, using $\Jk_0(w)$, can be found in 
\cite{belkin2004regularization,sindhwani200611}. We also mention 
related methodologies such as the support vector machine (SVM) 
\cite{bishop2007pattern} and robust convex minimization 
methods \cite{alaiz2016convex, alaiz2016robust} which may be based around minimization of 
$\Jk_0(w)$ with an additional soft penalty term; however since these
do not have a Bayesian interpretation we do not consider them here.  
Other forms of regularization have been considered such as the graph wavelet 
regularization \cite{shuman2011semi,hammond2011wavelets}. 

The underlying assumption in much of the work described in the previous paragraph
is that the labels are real-valued. An arguably more natural modelling assumption
is that there is a link function, such as the {\tt sign} function, connecting
a real-valued function on the graph nodes with the labels via thresholding. 
This way of thinking underlies
the probit approach \cite{williams1996gaussian} and the Bayesian level set
method \cite{iglesias2015bayesian,dunlop2016hierarchical}, both of which we will
study in this paper. Lying between the approaches initiated by
\cite{zhu2003semi} and those based on thesholding are the 
methods based on optimization over real-valued variables which are penalized from
taking values far from ${\pm 1}$; this idea was introduced
in the work of Bertozzi et al. \cite{bertozzi2012diffuse,van2012gamma}.
It is based on a Ginzburg-Landau relaxation of the 
discrete Total Variation (TV) functional, 
which  coincides with the graph cut energy.
This was generalized to multiclass classification
in \cite{garcia2014multiclass}. Following this line of work, several new
algorithms were developed for semi-supervised and unsupervised classification
problems on weighted graphs \cite{hu2015multi,merkurjev2013mbo, osting2014minimal}.

The Bayesian way of thinking is foundational in artificial
intelligence and machine learning research \cite{bishop2007pattern}. 
However, whilst the book 
\cite{williams1996gaussian} conducts a number of thorough
uncertainty quantification
studies for a variety of learning problems using Gaussian process
priors, most of the papers studying graph based learning referred to
above primarily use the Bayesian approach to learn hyperparameters 
in an optimization context, and do not consider uncertainty quantification. 
Thus the careful study of, and development of algorithms for,
uncertainty quantification in classification remains largely open.
There are a wide range of methodologies employed in the field of
uncertainty quantification, and the reader may consult the books
\cite{smith,sull,xiu} and the recent article \cite{owhadi} for
details and further references. Underlying all of these methods is
a Bayesian methodology which we adopt and which is attractive both for its clarity with
respect to modelling assumptions and its basis for application of
a range of computational tools. Nonetheless it is important to be
aware of limitations in this approach, in particular with regard
to its robustness with respect to the specification of the model,
and in particular the prior distribution on the unknown of
interest \cite{owhadi2}.

\subsection{Our Contribution}
In this paper, we focus exclusively on the problem of binary 
semi-supervised classification; however the methodology and conclusions
will extend beyond this setting, and to multi-class classification in particular.
Our focus is on a presentation which puts uncertainty quantification at the 
heart of the problem formulation, and we make four primary contributions:

\begin{itemize}

\item we define a number of different Bayesian formulations of the
graph-based semi-supervised learning problem and we connect them to
one another, to binary classification methods and to a variety 
of PDE-inspired approaches to classification; in so doing we 
provide a single framework for a variety of methods which have arisen
in distinct communities and we open up
a number of new avenues of study for the problem area; 

\item we highlight the pCN-MCMC method for posterior sampling 
which, based on analogies with its use for PDE-based inverse problems
\cite{CRSW08}, has the potential to 
sample the posterior distribution in a number of steps
which is independent of the number of graph nodes;

\item we introduce approximations exploiting the empirical properties of
the spectrum of the graph Laplacian, generalizing methods used in the
optimization context in \cite{bertozzi2012diffuse}, allowing for
computations at each MCMC step which scale well with respect to
the number of graph nodes;

\item we demonstrate, by means of numerical experiments on a range of
problems, both
the feasibility, and value, of Bayesian uncertainty quantification 
in semi-supervised, graph-based, learning, using the algorithmic ideas
introduced in this paper.

\end{itemize}

\subsection{Overview and Notation}
The paper is organized as follows. In section \ref{sec:2}, 
we give some background material needed for problem specification. 
In section \ref{sec:3}
we formulate the three Bayesian models used for the classification tasks. 
Section \ref{sec:4} introduces the MCMC and optimization algorithms that we
use. In section \ref{sec:5}, we present and discuss results of 
numerical experiments to illustrate our findings; these are
based on four examples of increasing size:
the house voting records from 1984 (as used in \cite{bertozzi2012diffuse}), 
the tunable two moons data set \cite{buhler2009spectral},
the MNIST digit data base \cite{lecun1998mnist}
and the hyperspectral gas plume imaging problem \cite{broadwater2011primer}.
We conclude in section \ref{sec:6}.
To aid the reader, we give here an overview of notation used 
throughout the paper. 
\begin{itemize}
\item  $Z$ the set of nodes of the graph, with cardinality $N$;
\item $Z'$ the set of nodes where labels are observed, with cardinality $J \le N$; 
\item $x:Z \mapsto \bbR^d$, feature vectors; 
\item $u: Z \mapsto \bbR$ latent variable characterizing nodes,
with $u(j)$ denoting evaluation of $u$ 
at node $j$; 
\item $S:\bbR \mapsto \{-1, 1\}$ the thresholding function; 
\item $\Se$ relaxation of $S$ using gradient flow 
in double-well potential $\We$;
\item $l: Z \mapsto  \{-1, 1\}$ the label value at each node with $l(j)=S(u(j));$
\item  $y: Z' \mapsto \{-1, 1\}$ or  $y: Z' \mapsto \bbR$, label data; 
\item $v: Z \mapsto \bbR$ with $v$ being a relaxation of the label variable $l$;
\item $A$ weight matrix of the graph, $L$ the resulting normalized 
 graph Laplacian;
 \item  $P$ the precision matrix and $C$ the covariance matrix, both found from $L$; 
\item $\{q_k,\lambda_k\}_{k=0}^{N-1}$ eigenpairs of $L$; 
\item $U$ :  orthogonal complement of the null space of the graph Laplacian 
$L$, given by $q_0^{\perp}$; 
\item $\GL$ : Ginzburg-Landau functional;
\item $\mu_0$, $\nu_0$ : prior probability measures; 
\item the measures denoted $\mu$ typically take argument $u$ and are real-valued; the measures denoted $\nu$ take argument $l$ on label space, or argument
$v$ on a real-valued relaxation of label space;
\end{itemize}

\section{Problem Specification}
\label{sec:2}
In subsection \ref{ssec:ssl} we formulate semi-supervised learning as 
a problem on a graph. Subsection \ref{ssec:graph} defines the relevant
properties of the graph Laplacian and in subsection \ref{ssec:gauss}
these properties are used to construct a Gaussian probability distribution; 
in section \ref{sec:3} this Gaussian will be used
to define our prior information
about the classification problem. In subsection \ref{ssec:thr} we discuss
thresholding which provides a link between the real-valued prior information,
and the label data provided for the semi-supervised learning task; in
section \ref{sec:3} this will be used to definine our likelihood.

\subsection{Semi-Supervised Learning on a Graph}
\label{ssec:ssl}

We are given a set of feature vectors $X = \{x(j), \dots, x(j), \dots, x(N)\}$
for each $j \in Z: = \{1, \dots, N\}.$
For each $j$ the feature vector $x(j)$
is an element of $\bbR^d$, so that $X \in \bbR^{d \times N}.$ 
Graph learning starts from the construction of an undirected graph $G$ with vertices
$Z$ and edge weights $\{A\}_{ij}=a_{ij}$ computed from the feature set  $X$. The
weights  $a_{ij}$ will depend only on $x(i)$ and $x(j)$ and will decrease 
monotonically with some measure of distance between $x(i)$ and $x(j)$; the weights
thus encode affinities between nodes of the graph. Although a very important modelling
problem, we do not discuss how to choose these weights in this paper.
For graph semi-supervised learning, we are also given a partial set of (possibly noisy) 
labels $y = \{y(j)|j\in Z'\}$, where $Z' \subseteq Z$ has size $J \le N$. 
The task is to infer the labels for all nodes in $Z$, 
using the weighted graph $G=(Z,A)$ and also the set of noisily observed labels $y$. 
In the Bayesian formulation which we adopt the feature set $X$, and hence 
the graph $G$, is viewed as prior information, describing correlations
amongst the nodes of the graph, and we combine this with a likelihood
based on the noisily observed labels $y$, 
to obtain a posterior distribution on the labelling of all nodes. 
Various Bayesian formulations, which differ in the specification of the observation model and/or the prior, are described in section \ref{sec:3}.
In the remainder of this section we give the background needed to
understand all of these formulations, thereby touching on the
graph Laplacian itself, its link to Gaussian probability distributions
and, via thresholding, to non-Gaussian probability distributions and to the
Ginzburg-Landau functional. An important point to appreciate is that building our
priors from Gaussians confers considerable
computational advantages for large graphs; for this reason the non-Gaussian
priors will be built from Gaussians via change of measure or push forward
under a nonlinear map.

\subsection{The Graph Laplacian}
\label{ssec:graph}

The graph Laplacian is central to many graph-learning algorithms. There
are a number of variants used in the literature;
see \modk{\cite{bertozzi2012diffuse, von2007tutorial}} for a discussion. We will work with the
normalized Laplacian, defined from the weight matrix $A=\{a_{ij}\}$ as follows.  
We define the diagonal matrix $D={\rm diag}\{d_{ii}\}$ with entries 
$d_{ii} = \sum_{j \in Z} a_{ij}.$ If we assume that the graph $G$ is 
connected, then $d_{ii}>0$ for all nodes $i \in Z$. We can then define the 
normalized graph Laplacian\footnote{In the majority of the paper
the only property of $L$ that we use is that it is symmetric positive semi-definite. We could therefore use other graph Laplacians, 
such as the unnormalized choice $L=D-A$, in most of the paper. The only
exception is the spectral approximation sampling
algorithm introduced later; that particular algorithm exploits empirical
properties of the symmetrized graph Laplacian. 
Note, though, that the choice of which graph Laplacian to
use can make a significant difference -- see \cite{bertozzi2012diffuse}, and Figure 2.1 therein. To make our exposition more concise we confine our presentation to the graph Laplacian \eqref{eq:s-lap}.}  as
\begin{equation}
\label{eq:s-lap}L=I-D^{-1/2}AD^{-1/2},
\end{equation} 
and the graph Dirichlet energy as  $J_0(u):=\frac12 \langle u,Lu \rangle$. Then
\begin{equation}
\label{eq:dirichlet}
J_0(D^{\frac12}u)=\frac{1}{4}\sum_{\{i,j\} \in Z \times Z} a_{ij}(u(i)-u(j))^2.
\end{equation}
Thus, similarly to the classical Dirichlet energy, this quadratic form
penalizes nodes from having different function values, with penalty
being weighted with respect to the similarity weights from $A$.  
Furthermore the identity shows that
$L$ is positive semi-definite. Indeed the vector of ones $\bbI$ 
is in the null-space of $D-A$ by construction, and hence $L$ 
has a zero eigenvalue with corresponding eigenvector $D^{\frac12}\bbI$.

We let $(q_k,\lambda_k)$ denote the eigenpairs of the matrix $L$,
so that\footnote{For the normalized graph Laplacian,
the upper bound $\lambda_{\rm max}=2$ may be found in 
\cite[Lemma 1.7, Chapter 1]{chung1997spectral}; but with further structure on the
weights many spectra saturate at $\lambda_{\rm max} \approx 1$ 
(see Appendix), a fact we will exploit later.} 
\begin{equation}
\label{eq:mz}
\lambda_0 \le \lambda_1 \le \dots \le \lambda_{N-1} \le \lambda_{\rm max}<\infty, \quad \langle q_j,q_k\rangle=\delta_{jk}.
\end{equation}
The eigenvector corresponding to $\lambda_0=0$
is $q_0=D^{\frac12}\bbI$ and $\lambda_1>0$,
assuming a fully connected graph. Then 
$L=Q\Lambda Q^*$ where $Q$ has columns $\{q_k\}_{k=0}^{N-1}$ and 
$\Lambda$ is a diagonal
matrix with entries $\{\lambda_k\}_{k=0}^{N-1}$. Using these eigenpairs  the graph Dirichlet energy can be written as 
\begin{equation}
\label{eq:eigdir}
\frac{1}{2}\langle u,L u \rangle=\frac{1}{2}\sum_{j=1}^{N-1} \lambda_j (\langle u, q_j \rangle )^2; 
\end{equation}
this is analogous to decomposing the classical Dirichlet energy using
Fourier analysis.

\subsection{Gaussian Measure}
\label{ssec:gauss}
We now show how to build a Gaussian distribution with 
negative log density proportional to $J_0(u).$
Such a Gaussian prefers functions that have larger components 
on the first few eigenvectors of the graph Laplacian, where the eigenvalues
of $L$ are smaller. The corresponding eigenvectors
carry rich geometric information about the weighted graph. 
For example, the second eigenvector of $L$ is the  
{\em Fiedler vector} and solves a relaxed normalized min-cut problem
\cite{von2007tutorial,higham2004unified}. The Gaussian distribution thereby 
connects geometric intuition embedded within the graph Laplacian to
a natural probabilistic picture.

To make this connection concrete
we define diagonal matrix $\Sigma$ with entries defined
by the vector
$$(0,\lambda_1^{-1}, \cdots, \lambda_{N-1}^{-1})$$
and define the positive semi-definite covariance matrix $C =c Q \Sigma Q^*$; choice
of the scaling $c$ will be discussed below.
We let $\mu_0:=\cN(0,C).$
Note that the covariance matrix is that of a Gaussian with variance proportional
to $\lambda_j^{-1}$ in direction $q_j$ thereby leading to structures which
are more likely to favour the Fiedler vector ($j=1$), and lower values of $j$
in general, than it does for higher values. The fact that the first eigenvalue of
$C$ is zero ensures that any draw from $\mu_0$ changes sign, because it will be
orthogonal to $q_0.$\footnote{Other treatments of the first eigenvalue are 
possible and may be useful but for simplicity of exposition 
we do not consider them in this paper.} 
To make this intuition explicit we recall the Karhunen-Loeve
expansion which constructs  a sample $u$ from the Gaussian $\mu_0$ according to 
the random sum
\begin{equation}
\label{eq:KL0}
u=c^{\frac12}\sum_{j=1}^{N-1} \lambda_j^{-\frac12} q_j z_j,
\end{equation}
where the $\{z_j\}$ are i.i.d. $\cN(0,1).$
Equation \eqref{eq:mz} thus implies that $\langle u,q_0 \rangle=0.$

We choose the constant of proportionality $c$ 
as a rescaling which enforces the property
$\bbE |u|^2=N$ for $u \sim \mu_0:=\cN(0,C)$; 
in words the per-node variance is $1$. 
Note that, using the orthogonality of the $\{q_j\}$, 
\begin{equation}
\label{eq:c2}
\bbE |u|^2=c\sum_{j=1}^{N-1}\lambda_j^{-1}\bbE z_j^2=c\sum_{j=1}^{N-1}\lambda_j^{-1}
\quad\Longrightarrow\quad c=N\Bigl(\sum_{j=1}^{N-1}\lambda_j^{-1}\Bigr)^{-1}.
\end{equation}
We reiterate that the support of the measure $\mu_0$ is 
the space $U:=q_0^{\perp}={\rm span}\{q_1, \cdots, q_{N-1}\}$
and that, on this space, the probability density function is
proportional to $$\exp\Bigl(-c^{-1}J_0(u)\Bigr)=\exp\Bigl(-\frac{1}{2c}\langle u,Lu \rangle \Bigr),$$
so that the \emph{precision matrix} of the Gaussian is $P=c^{-1}L.$
In what follows the sign of $u$ will be related to the classification; 
since all the entries of $q_{0}$ are positive, working on the space $U$ 
ensures a sign change in $u$, and hence a non-trivial classification.

\subsection{Thresholding and Non-Gaussian Probability Measure}
\label{ssec:thr}
For the models considered in this paper, the label space of the problem is discrete while the latent variable $u$ through which we will capture the correlations
amongst nodes of the graph, encoded in the feature vectors, is real-valued. 
We describe thresholding, and a relaxation of thresholding,
to address the need to connect these two differing sources of information about the problem. 
In what follows the latent variable $u: Z \to \bbR$ (appearing in the probit and Bayesian
level set methods) is thresholded to obtain 
the label variable $l: Z \to \{-1,1\}.$ The variable $v: Z \to \bbR$
(appearing in the Ginzburg-Landau method)
is a real-valued relaxation of the label variable $l.$
The variable $u$ will be endowed with a Gaussian probability distribution. From this
the variable $l$ (which lives on a discrete space) and $v$ (which is real-valued, but
concentrates near the discrete space supporting $l$) will be endowed with 
non-Gaussian probability distributions.

Define the (signum) function $S:\bbR \mapsto \{-1,1\}$ by
$$S(u)=1, \;u \ge  0\quad{\rm and}\quad S(u)=-1, \; u<0.$$
This will be used to connect the latent variable
$u$ with the label variable $l$.
The function $S$ may be relaxed by defining
$\Se(u)=v|_{t=1}$ where $v$ solves the gradient flow
$$\dot{v}=-\nabla \We(v), \quad v|_{t=0}=u\quad\quad{\rm for}\,\,{\rm potential}
\quad\quad   \We(v) = \frac{1}{4\epsilon}(v^2 - 1)^2.$$
This will be used, indirectly, to connect the latent variable $u$ with 
the real-valued relaxation of the label variable, $v$.
Note that $\Se(\cdot) \to S(\cdot)$, pointwise, as $\epsilon \to 0$, on
$\bbR\backslash\{0\}.$ This reflects the fact that the gradient flow minimizes
$\We$, asymptotically as $t \to \infty$, 
whenever started on $\bbR\backslash\{0\}.$

We have introduced a Gaussian measure $\mu_0$ on the latent variable $u$
which lies in $U \subset \bbR^N$; we 
now want to introduce two ways of constructing non-Gaussian measures on 
the label space $\{-1,1\}^{N},$ or on real-valued relaxations of label
space, building on the measure $\mu_0.$ 
The first is to consider the
push-forward of measure $\mu_0$ under the map $S$: $S^{\sharp}\mu_0.$
When applied to a sequence $l:Z \mapsto \{-1, 1\}^N$ this gives
$$\bigl(S^{\sharp}\mu_0\bigr)(l)=\mu_0\Bigl( \{u|S(u(j))=l(j), \forall 1 \le j \le N\} \Bigr),$$
recalling that $N$ is the cardinality of $Z$. The definition is readily 
extended to components of $l$ defined only on subsets of $Z$.
Thus $S^{\sharp}\mu_0$ is a measure on the label
space $\{-1, 1\}^N.$ 
The second approach is to work with a change of measure
from the Gaussian $\mu_0$ in such a way
that the probability mass on $U \subset \bbR^N$ concentrates close 
to the label space $\{-1, 1\}^N$. We may achieve this by defining the measure 
$\nu_0$ via its Radon-Nykodim derivative 
\begin{equation}
\label{eq:gl-meas}
\frac{d\nu_0}{d\mu_0}(v) \propto e^{-\sum_{j\in Z} \We(v(j))}.
\end{equation}
We name $\nu_0$ the Ginzburg-Landau measure, since the negative log density 
function of $\nu_0$ is the graph Ginzburg-Landau functional  
\begin{equation}
\label{eq:gl-dist}
\GL(v):=\frac{1}{2c} \langle v, Lv \rangle +\sum_{j\in Z} \We(v(j)). 
\end{equation}
The Ginzburg-Landau distribution defined by $\nu_0$ 
can be interpreted as a non-convex ground 
relaxation of the discrete MRF model \cite{zhu2005semi}, in contrast to the convex relaxation 
which is the Gaussian Field \cite{zhu2003semi}. 
Since the double well has minima at the label 
values $\{-1, 1\}$, the probability mass of $\nu_0$ is concentrated near 
the modes $\pm 1$, and $\epsilon$ controls this concentration effect.

\section{Bayesian Formulation}
\label{sec:3}
\label{sec:problem-spec}
In this section we formulate three different  Bayesian models for
the semi-supervised learning problem. 
The three models all combine the ideas described in the previous section
to define three distinct posterior distributions. It is important to
realize that these different models will give different answers to
the same questions about uncertainty quantification.
The choice of which Bayesian model to use is related to the data itself, and making
this choice is beyond the scope of this paper. Currently the choice
must be addressed on a case by case basis, as is done when choosing
an optimization method for classification. 
Nonetheless we will demonstrate that the shared
structure of the three models means that a common algorithmic framework
can be adopted and we will make some conclusions about the relative costs
of applying this framework to the three  models.

We denote the latent variable by $u(j)$, $j \in Z$, the thresholded value of $u(j)$ by $l(j)=S(u(j))$ which is interpreted as 
the label assignment at each node $j$, and noisy observations of the 
binary labels by 
$y(j)$, $j \in Z'$.  
The variable $v(j)$ will be used to denote the real-valued relaxation of
$l(j)$ used for the Ginzburg-Landau model.
Recall Bayes formula which transforms a prior density $\bbP(u)$
on a random variable $u$ into a posterior density $\bbP(u|y)$ 
on the conditional random variable $u|y$:
$$\bbP(u|y)=\frac{1}{\bbP(y)}\bbP(y|u)\bbP(u).$$
We will now apply this formula to condition our graph latent variable $u$,
whose thresholded values correspond to labels, on the noisy label data $y$
given at $Z'$. As prior on $u$ we will always use $\bbP(u)du=\mu_0(du);$ we will\
describe two different likelihoods.
We will also apply the formula to condition relaxed label variable $v$, on the same
label data $y$, via the formula
$$\bbP(v|y)=\frac{1}{\bbP(y)}\bbP(y|v)\bbP(v).$$
We will use as prior the non-Gaussian $\bbP(v)dv=\nu_0(dv).$

For the probit and level-set models we now explicitly state the prior 
density $\bbP(u)$,
the likelihood function $\bbP(y|u)$, and the posterior density $\bbP(u|y)$;
in the Ginzburg-Landau case $v$ will replace $u$ and we
will define the densities $\bbP(v), \bbP(y|v)$ and $\bbP(v|y)$. 
Prior and posterior probability
measures associated with letter $\mu$ are on the latent variable $u$;
measures associated with letter $\nu$ are on the label space, or real-valued
relaxation of the label space.

\subsection{Probit}
\label{subsec:probit}

The probit method is designed for classification and is described in
\cite{williams1996gaussian}. In that context Gaussian process priors
are used and, unlike the graph Laplacian construction used here, 
do not depend on the unlabel data. Combining Gaussian process priors and
graph Laplacian priors was suggested and studied in \cite{belkin2006manifold,sindhwani200611,lu2011multi}.  A recent fully Bayesian treatment of the
methodology using unweighted graph Laplacians 
may be found in the paper \cite{hartog2016nonparametric}. 
In detail our model is as follows.

\noindent{\bf Prior} We take as prior on $u$ the Gaussian $\mu_0.$
Thus 
$$\bbP(u) \propto \exp\Bigl(-\frac12 \langle u, Pu \rangle \Bigr).$$

\noindent{\bf Likelihood} For any $j \in Z'$
$$y(j)=S\Bigl(u(j)+\eta(j)\Bigr)$$ 
with the $\eta(j)$ drawn i.i.d from $\cN(0,\gamma^2).$
We let
$$\Psi(v;\gamma)=\frac{1}{\sqrt{2\pi\gamma^2}}\int_{-\infty}^v \exp\big(-t^2/2\gamma^2\bigr)dt,$$
the cumulative distribution function (cdf) of $\cN(0,\gamma^2)$,
and note that then
$$\bbP\bigl(y(j)=1|u(j)\bigr)=\bbP\bigl(\cN(0,\gamma^2)>-u(j)\bigr)=\Psi(u(j);\gamma)=\Psi(y(j)u(j);\gamma);$$
similarly
$$\bbP\bigl(y(j)=-1|u(j)\bigr)=\bbP\bigl(\cN(0,\gamma^2)<-u(j)\bigr)=\Psi(-u(j);\gamma)=\Psi(y(j)u(j);\gamma).$$ 

\noindent{\bf Posterior}
Bayes' Theorem gives posterior $\mup$ with probability density function (pdf)
$$\Pp(u|y) \propto 
\exp\Bigl(-\frac12 \langle u, Pu \rangle -\Phip(u;y)\Bigr)$$
where
$$\Phip(u;y):=
-\sum_{j \in Z'}\log \bigl(\Psi(y(j)u(j);\gamma)\bigr).$$
We let $\nup$ denote the push-forward under $S$ of $\mup: \nup=S^\sharp \mup.$

\noindent{\bf MAP Estimator}
This is the minimizer of the negative of the log posterior. Thus
we minimize the following objective function over $U$:
$$\Jp(u)=\frac12 \langle u, Pu \rangle-\sum_{j \in Z'}
{\rm log} \Bigl(\Psi(y(j)u(j);\gamma)\Bigr).$$
This is a convex function, a fact which is well-known in related contexts, 
but which we state and prove in Proposition 1 Section B of the appendix for the sake of completeness.
In view of the close relationship between this problem and the
level-set formulation described next, for which there are no minimizers,
we expect that minimization may not be entirely straightforward in the
$\gamma \ll 1$ limit. This is manifested in the presence of near-flat regions 
in the probit log likelihood function when $\gamma \ll 1$.

Our variant on the probit methodology differs from that in
\cite{hartog2016nonparametric} in several ways: (i) our
prior Gaussian is scaled to have per-node variance one, whilst in
\cite{hartog2016nonparametric} the per node variance is a hyper-parameter
to be determined; (ii) our prior is supported on $U=q_0^{\perp}$ whilst
in \cite{hartog2016nonparametric} the prior precision is found by shifting
$L$ and taking a possibly fractional power of the resulting matrix, resulting
in support on the whole of $\bbR^N$; (iii) we allow for a scale parameter
$\gamma$ in the observational noise, whilst in \cite{hartog2016nonparametric}
the parameter $\gamma=1.$

\subsection{Level-Set}
\label{subsec:levelset}
This method is designed for problems considerably more general 
than classification on a graph \cite{iglesias2015bayesian}. For the current application, this model is exactly the same as probit except for the order in which the noise $\eta(j)$ and the thresholding function $S(u)$ is applied in
the definition of the data. 
Thus we again take as {\bf Prior} for $u$, the Gaussian $\mu_0.$
Then we have:

\noindent{\bf Likelihood} For any $j \in Z'$
$$y(j)=S\bigl(u(j)\bigr)+\eta(j)$$ 
with the $\eta(j)$ drawn i.i.d from $\cN(0,\gamma^2).$
Then
$$\bbP\Bigl(y(j)|u(j)\Bigr) \propto \exp\Bigl(-\frac{1}{2\gamma^2}
|y(j)-S\bigl((u(j)\bigr)|^2\Bigr).$$

\noindent{\bf Posterior} Bayes' Theorem gives posterior $\mul$ with pdf
$$\Pl(u|y) \propto \exp\Bigl(-\frac12 \langle u, Pu \rangle -
\Phil(u;y)\Bigr)$$
where
$$\Phil(u;y)=\sum_{j \in Z'} \Bigl(\frac{1}{2\gamma^2}
|y(j)-S\bigl(u(j)\bigr)|^2\Bigr).$$ 
We let $\nul$ denote the pushforward under $S$ of $\mul: \nul=S^\sharp \mul.$

\noindent{\bf MAP Estimator Functional}
The negative of the log posterior is, in this case, given by
$$\Jl(u)=\frac12 \langle u, Pu \rangle+\Phil(u;y).$$
However, unlike the probit model, the Bayesian level-set 
method has no MAP estimator -- the infimum of $\Jl$ is not attained and
this may be seen by noting that, if the infumum was attained at any
non-zero point $u^\star$ then $\epsilon u^\star$ would reduce the
objective function for any $\epsilon \in (0,1)$; however the point
$u^\star=0$ does not attain the infimum.
This proof is detailed in \cite{iglesias2015bayesian} for a closely related 
PDE based model, and the proof is easily adapted.

\subsection{Ginzburg-Landau}
\label{subsec:gl}
For this model, we take as prior the Ginzburg-Landau measure $\nu_0$
defined by (\ref{eq:gl-meas}),
and employ a Gaussian likelihood for the observed labels. This construction gives the Bayesian posterior whose MAP estimator is the objective function introduced and studied in \cite{bertozzi2012diffuse}.

\noindent{\bf Prior} We define prior on $v$ to be the Ginzburg-Landau measure 
$\nu_0$ given by (\ref{eq:gl-meas}) with density
$$\bbP(v) \propto e^{-\GL(v)}.$$

\noindent{\bf Likelihood} For any $j \in Z'$
$$y(j)=v(j)+\eta(j)$$ 
with the $\eta(j)$ drawn i.i.d from $\cN(0,\gamma^2).$
Then
$$\bbP\Bigl(y(j)|v(j)\Bigr) \propto \exp\Bigl(-\frac{1}{2\gamma^2}
|y(j)-v(j)|^2\Bigr).$$

\noindent{\bf Posterior} Recalling that $P=c^{-1}L$ we see that
Bayes' Theorem gives posterior $\nug$ with pdf
\begin{eqnarray*}
\Pg(v|y) &\propto& \exp\Bigl(-\frac12 \langle v, Pv \rangle - \Phig(v;y)\Bigr),\\
\Phig(v;y)&:=&\sum_{j \in Z} W_\epsilon\bigl(v(j)\bigr)+\sum_{j \in Z'} \Bigl(\frac{1}{2\gamma^2}
|y(j)-v(j)|^2\Bigr) \Bigr).
\end{eqnarray*}

\noindent{\bf MAP Estimator}
This is the minimizer of the negative of the log posterior. Thus
we minimize the following objective function over $U$:
$$\Jg(v)=\frac12 \langle v, Pv \rangle+\Phig(v;y).$$
This objective function was introduced in \cite{bertozzi2012diffuse}
as a relaxation of the min-cut problem, penalized by data; the
relationship to min-cut was studied rigorously in
\cite{van2012gamma}.
The minimization problem for $\Jg$ is non-convex and has multiple
minimizers, reflecting the combinatorial character of the
min-cut problem of which it is a relaxation.

\subsection{Uncertainty Quantification for Graph Based Learning}
\label{ssec:uq}

In Figure \ref{fig:spectral-approx}
we plot the component of the negative log likelihood at a labelled
node $j$, as a function of the latent variable $u=u(j)$
with data $y=y(j)$ fixed, for the probit and Bayesian level-set
models. The log likelihood for the Ginzburg-Landau 
formulation is not directly comparable as it is a function of the 
relaxed label variable $v(j)$, with
respect to which it is quadratic with minimum at the data point $y(j).$

\begin{figure}[!ht]
\centering
\includegraphics[height=4.2cm, width=6cm]{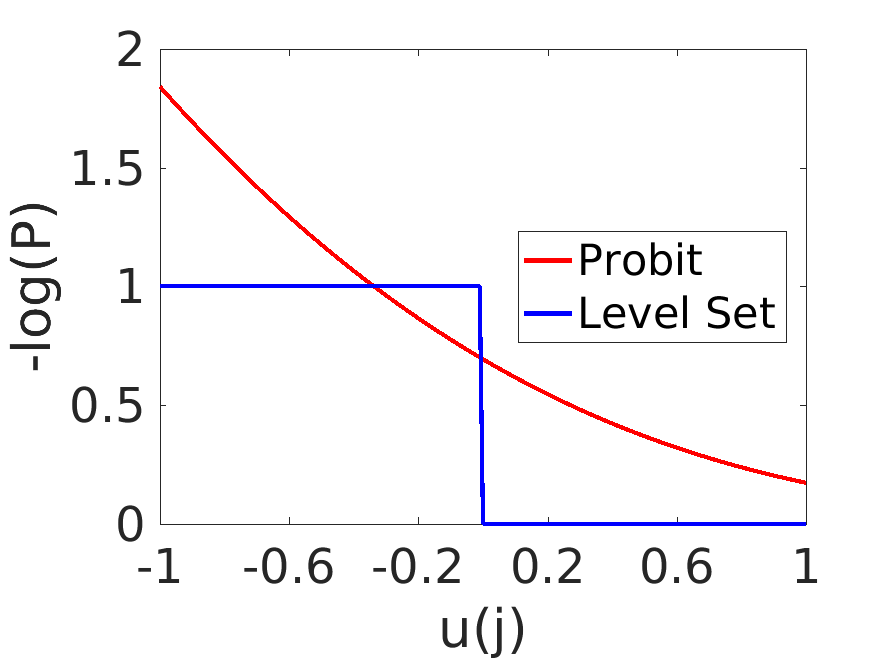}
\caption{Plot of a component of the negative log likelihood for a fixed node $j$. We set $\gamma = 1/\sqrt{2}$ for probit and Bayesian level-set. Since $\Phi(u(j); 1) = \Phi(-u(j); -1)$ for probit and Bayesian level-set, we omit the plot for $y(j) = -1$.} \label{fig:spectral-approx}
\end{figure}

The probit and Bayesian level-set models lead to posterior
distributions $\mu$ (with different subscripts) in latent variable
space, and pushforwards under $S$, denoted $\nu$ (also with different subscripts),
in label space. The Ginzburg-Landau formulation leads to a measure
$\nug$ in (relaxed) label space. Uncertainty quantification for semi-supervised
learning is
concerned with completely characterizing these posterior distributions. 
In practice this may be achieved by sampling using MCMC methods. In this
paper we will study four measures of uncertainty:

\begin{itemize}

\item we will study the empirical pdfs of the latent and label variables
at certain nodes;

\item we will study the posterior mean of the label variables at certain nodes; 

\item we will study the posterior variance of the label variables averaged over
all nodes;

\item we will use the posterior mean or variance to order nodes into those whose 
classifications are most uncertain and those which are most certain.

\end{itemize}

For the probit and Bayesian level-set models we interpret the thresholded 
variable $l = S(u)$ as the binary label assignments corresponding to a 
real-valued configuration $u$; for Ginzburg-Landau we may simply take $l=v$
as the model is posed on (relaxed) label space. The node-wise posterior mean of 
$l$ can be used as a useful confidence score of the class assignment of 
each node. The node-wise posterior mean $s^l_{j}$ is defined as 
\begin{equation}
\label{eq:posterior-mean}
s^l_{j} := \mathbb{E}_{\nu}(l(j)), 
\end{equation}
with respect to any of the posterior measures $\nu$ in label space. 
Note that for probit and Bayesian level set $l(j)$ is a binary random variable
taking values in $\{\pm 1\}$ and we have $s^l_{j} \in [-1,1].$ 
In this case if $q=\nu(l(j)=1)$ then $q=\frac12(1+s^l_{j})$.
Furthermore
\[
{\rm Var}_{\nu}(l(j))=4q(1-q)=1-(s^l_{j})^2. 
\]
Later we will find it useful to consider the variance averaged over all nodes
and hence define\footnote{Strictly speaking ${\rm Var}(l)=N^{-1}{\rm Tr}\bigl({\rm Cov}(l)\bigr).$}
\begin{equation}
\label{eq:mean-posterior-variance}
{\rm Var}(l) = \frac{1}{N} \sum_{j=1}^N{\rm Var}_{\nu}(l(j)).
\end{equation}
Note that the maximum value obtained by {\rm Var}(l) is $1$.
This maximum value is attained under the Gaussian prior $\mu_0$ that 
we use in this paper. The deviation from this maximum under
the posterior is a measure of the information content of the
labelled data. Note, however, that the prior does contain information
about classifications, in the form of correlations between vertices;
this is not captured in \eqref{eq:mean-posterior-variance}.


\section{Algorithms}
\label{sec:4}
From Section \ref{sec:3}, we see that for all of the models considered, 
the posterior $\bbP(w|y)$ has the form 
$$\bbP(w|y) \propto \exp\bigl(-\Jk(w)\bigr), \quad \Jk(w)=\frac12 \langle w, Pw \rangle + \Phi(w))$$ for some function 
$\Phi$, different for each of the three models (acknowledging that in
the Ginzburg-Landau case the independent variable is $w=v$, real-valued
relaxation of label space, whereas for the other models $w=u$ an underlying latent variable
which may be thresholded by $S(\cdot)$ into label space.) Furthermore, the MAP
estimator is the minimizer of $\Jk.$ Note that $\Phi$ is
differentiable for the Ginzburg-Landau and probit models, but not for the 
level-set model. 
We introduce algorithms for both sampling (MCMC) and MAP estimation
(optimization)  
that apply in this general framework. The sampler we employ does not use
information about the gradient of $\Phi$; the MAP estimation algorithm
does, but is only employed on the Ginzburg-Landau and probit models. 
Both sampling and optimization algorithms 
use spectral properties of the precision matrix $P$, which
is proportional to the graph Laplcian $L$.

\subsection{MCMC}
Broadly speaking there are two strong competitors as samplers for this problem:
Metropolis-Hastings based methods, and Gibbs based samplers.
In this paper we focus entirely on Metropolis-Hastings methods
as they may be used on all three models considered here.
In order to induce scalability with respect to size of $Z$
we use the preconditioned Crank-Nicolson (pCN) method described in \cite{CRSW08}
and introduced in the context of diffusions by  Beskos et. al. in \cite{BRSV08}
and by Neal in the context of machine learning \cite{neal}.
The method is also robust with respect to the small noise limit $\gamma \to 0$
in which the label data is perfect. The pCN based
approach is compared with Gibbs like methods for probit, to which they both apply,
in \cite{luos}; both large data sets $N \to \infty$ and small noise $\gamma \to 0$
limits are considered.

The standard random walk Metropolis (RWM) algorithm 
suffers from the fact that the optimal proposal variance or stepsize 
scales inverse proportionally to the dimension of the state space \cite{R}, 
which is the graph size $N$ in this case. The pCN method is designed so that 
the proposal variance required to obtain a given
acceptance probability scales independently of the dimension of the 
state space (here the number of graph nodes $N$), hence in practice 
giving faster convergence of the MCMC when compared with 
RWM \cite{beskos2009optimal}. We restate the pCN method 
as Algorithm \ref{alg:pcnfull}, and then follow with various variants
on it in  Algorithms \ref{alg:pcneig} and \ref{alg:pcneigfill}. 
In all three algorithms $\beta \in [0,1]$ is the key parameter which
determines the efficiency of the MCMC method: small $\beta$ leads to high 
acceptance probability but small moves; large $\beta$ leads to low acceptance
probability and large moves. Somewhere between these extremes is an
optimal choice of $\beta$ which minimizes the asymptotic variance of
the algorithm when applied to compute a given expectation.

\begin{algorithm}[!ht]
\caption{pCN Algorithm}
\label{alg:pcnfull}
\begin{algorithmic}[1]
\STATE{Input: $L$. $\Phi(u)$. $u^{(0)} \in U$.}
\STATE{Output: $M$ Approximate samples from the posterior distribution}
\STATE{Define: $\alpha(u,w) = \min\{1, \exp(\Phi(u) - \Phi(w)\}$.} 
\WHILE{$k < M$ }
\STATE{ $w^{(k)} = \sqrt{1 - \beta^2}u^{(k)} + \beta \xi^{(k)}$, where $\xi^{(k)} \sim \cN(0, C)$ via equation (\ref{eq:KL}). }
\STATE{Calculate acceptance probability $\alpha(u^{(k)}, w^{(k)})$.}
\STATE{Accept $w^{(k)}$ as $u^{(k+1)}$ with probability $\alpha(u^{(k)}, w^{(k)})$, otherwise $u^{(k+1)} = u^{(k)}$. }
\ENDWHILE
\end{algorithmic}
\end{algorithm}

The value $\xi^{(k)}$ is a sample from the prior $\mu_0$. If the eigenvalues
and eigenvectors of $L$ are all known then the Karhunen-Loeve 
expansion \eqref{eq:KL} gives 
\begin{equation}
\label{eq:KL}
\xi^{(k)}=c^{\frac12}\sum_{j=1}^{N-1} \lambda_j^{-\frac12} q_j z_j,
\end{equation}
where $c$ is given by \eqref{eq:c2}, the $z_j, j = 1\dots N-1$ are i.i.d 
centred unit Gaussians and the equality is in law.

\subsection{Spectral Projection}

For graphs with a large number of nodes $N$, it is prohibitively
costly to directly sample from the distribution $\mu_0$, since doing
so involves knowledge of a complete eigen-decomposition of $L$, in
order to employ \eqref{eq:KL}. A method that is frequently used in 
classification tasks is to restrict the support of $u$ to the eigenspace 
spanned by the first $\ell$
eigenvectors with the smallest non-zero eigenvalues of $L$ (hence
largest precision)
and this idea may be used to approximate the pCN method; this leads to a low
rank approximation. In particular
we approximate samples from $\mu_0$ by
\begin{equation}
\label{eq:KL2}
\xi^{(k)}_{\ell}=c_\ell^{\frac12}\sum_{j=1}^{\ell - 1} \lambda_j^{-\frac12} q_jz_j,
\end{equation}
where $c_\ell$ is given by \eqref{eq:c2} truncated after $j=\ell-1$, the $z_j$ are i.i.d 
centred unit Gaussians and the equality is in law.
This is a sample from $\cN(0,C_{\ell})$ where 
$C_{\ell}=c_{\ell} Q\Sigma_{\ell}Q^*$
and the diagonal entries of $\Sigma_{\ell}$ are set to zero for the entries
after $\ell.$
In practice, to implement this algorithm, it is only necessary to compute
the first $\ell$ eigenvectors of the graph Laplacian $L$.
This gives Algorithm \ref{alg:pcneig}.

\begin{algorithm}[!ht]
\caption{pCN Algorithm With Spectral Projection}
\label{alg:pcneig}
\begin{algorithmic}[1]
\STATE{Input: $L$. $\Phi(u)$. $u^{(0)} \in U$.}
\STATE{Output: $M$ Approximate samples from the posterior distribution}
\STATE{Define: $\alpha(u,w) = \min\{1, \exp(\Phi(u) - \Phi(w)\}$.} 
\WHILE{$k < M$ }
\STATE{ $w^{(k)} = \sqrt{1 - \beta^2}u^{(k)} + \beta \xi^{(k)}_{\ell}$, where $\xi^{(k)}_{\ell} \sim \cN(0, C_{\ell})$ via equation (\ref{eq:KL2}). }
\STATE{Calculate acceptance probability $\alpha(u^{(k)}, w^{(k)})$.}
\STATE{Accept $w^{(k)}$ as $u^{(k+1)}$ with probability $\alpha(u^{(k)}, w^{(k)})$, otherwise $u^{(k+1)} = u^{(k)}$. }
\ENDWHILE
\end{algorithmic}
\end{algorithm}

\subsection{Spectral Approximation}
Spectral projection often leads to good classification results, but may lead 
to reduced posterior variance and a posterior distribution that is overly smooth on the graph domain. 
We propose an improvement on the method that preserves the variability of the 
posterior distribution but still only involves calculating the first $\ell$
eigenvectors of $L.$ This is based on the empirical observation that in
many applications the spectrum of $L$ saturates and satisfies, for $j \ge \ell$, 
$\lambda_j \approx \bar{\lambda}$ for some $\bar{\lambda}$. Such behaviour
may be observed in b), c) and d) of Figure \ref{fig:spectra-all}; 
in particular note that in the hyperspectal case $\ell \ll N$. 
We assume such behaviour in deriving the low rank approximation used
in this subsection. (See Appendix for a detailed discussion of the graph Laplacian spectrum.)
We define $\Sigma_{\ell,o}$ by overwriting the diagonal
entries of $\Sigma$ from $\ell$ to $N-1$ with $\bar{\lambda}^{-1}.$ 
We then set $C_{\ell,o}=c_{\ell,o}Q \Sigma_{\ell,o}Q^*$, 
and generate samples from $\cN(0,C_{\ell,o})$ (which are 
approximate samples from $\mu_0$) by setting
\begin{equation}
\label{eq:KL3}
\xi^{(k)}_{\ell,o}=c_{\ell, o}^{\frac12}\sum_{j=1}^{\ell-1} \lambda_j^{-\frac12} q_jz_j+ c_{\ell, o}^{\frac12}\bar{\lambda}^{-\frac12} \sum_{j=\ell}^{N-1}q_jz_j,
\end{equation}
where $c_{\ell, o}$ is given by \eqref{eq:c2} with $\lambda_j$ replaced
by $\bar{\lambda}$ for $j\geq \ell$, the $\{z_j\}$ are centred
unit Gaussians, and the equality is in law.
Importantly samples according to \eqref{eq:KL3} can be
computed very efficiently. In particular there is no need to compute $q_j$ 
for $j\geq \ell$, and the quantity $\sum_{j=\ell}^{N-1}q_jz_j$ can be computed 
by first taking a sample $\bar{z} \sim \cN(0, I_N)$, and then projecting 
$\bar{z}$ onto $U_{\ell}:={\rm span}(q_\ell, \dots, q_{N-1})$. Moreover, projection onto 
$U_{\ell}$ can be computed only using $\{q_1, \dots, q_{\ell-1}\}$, since 
the vectors span the orthogonal complement of $U_{\ell}$. Concretely, we have 
$$\sum_{j=\ell}^{N-1}q_jz_j = \bar{z} - \sum_{j=1}^{\ell-1} q_j\langle q_j, \bar{z}\rangle,$$ where $\bar{z} \sim \cN(0, I_N)$ and equality is in law. 
Hence the samples $\xi^{(k)}_{\ell,o}$ can be computed by  
\begin{equation}
\label{eq:Eigfill-Samples}
\xi^{(k)}_{\ell,o}=c_{\ell,o}^{\frac12}\sum_{j=1}^{\ell-1} \lambda_j^{-\frac12} q_jz_j+ c_{\ell,o}^{\frac12}\bar{\lambda}^{-\frac12}\Bigl(\bar{z} - \sum_{j=1}^{\ell-1} q_j\langle q_j, \bar{z}\rangle\Bigr).
\end{equation}
The vector $\xi^{(k)}_{\ell,o}$ is a sample from $\cN(0,C_{\ell,o})$ 
and results in Algorithm \ref{alg:pcneigfill}. 
Under the stated empirical properties of the graph Laplacian,
we expect this to be a better approximation of the prior covariance structure
than the approximation of the previous subsection.

\begin{figure}[!ht]
\centering
\subfloat[MNIST49]{\includegraphics[width=35mm]{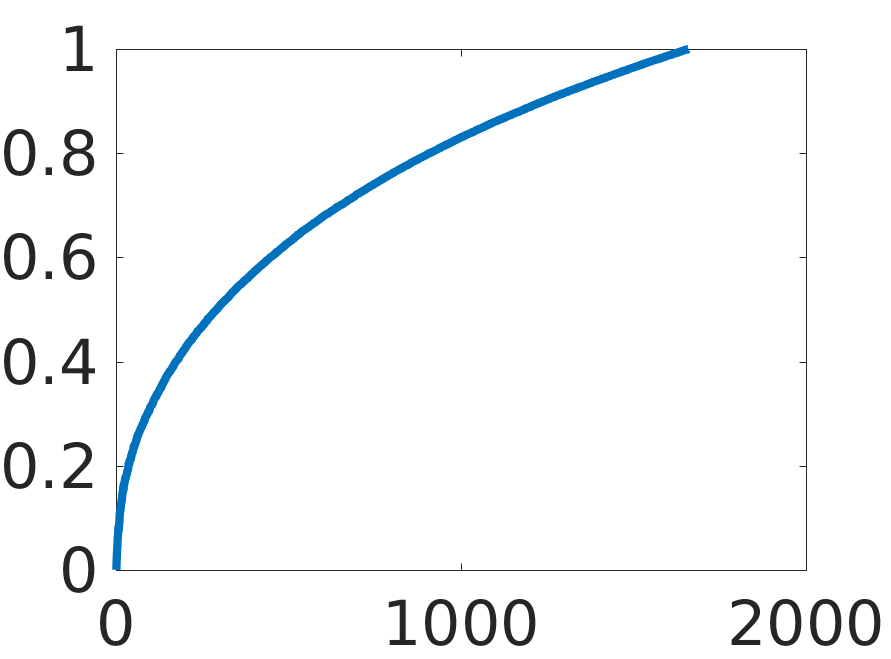}}
\subfloat[Two Moons]{\includegraphics[width=35mm]{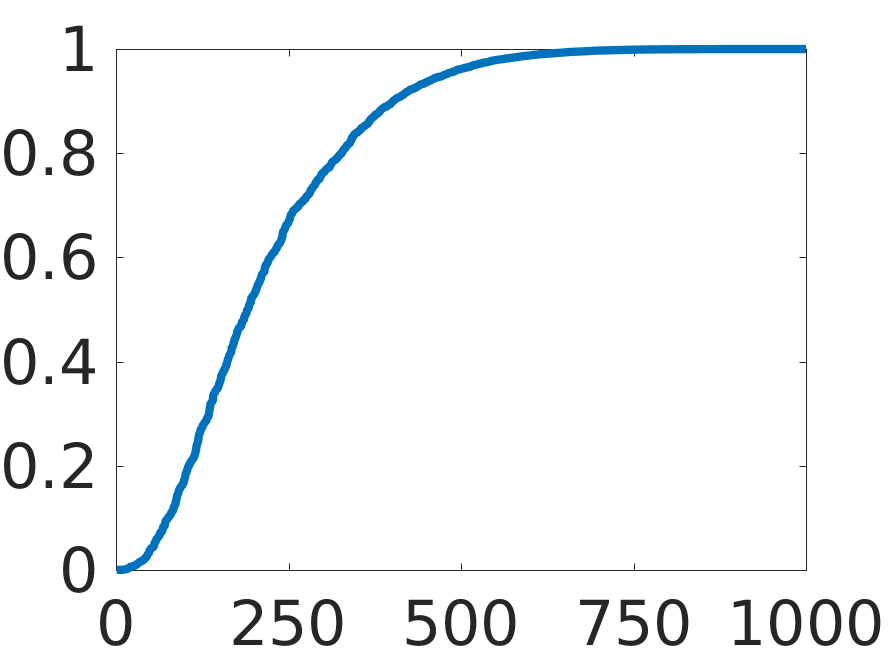}}
\subfloat[Hyperspectral]{\includegraphics[width=35mm]{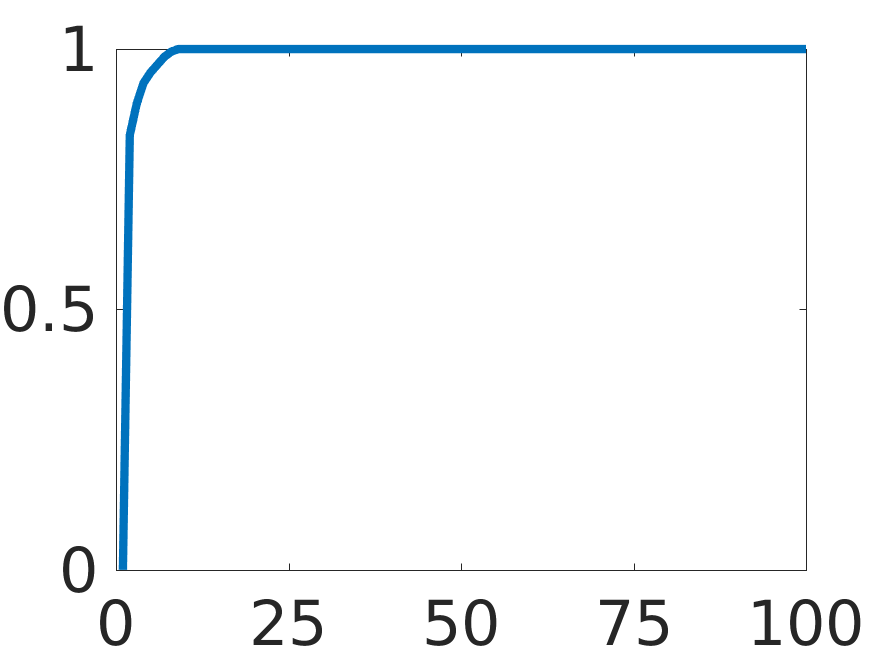}}
\subfloat[Voting Records]{\includegraphics[width=35mm]{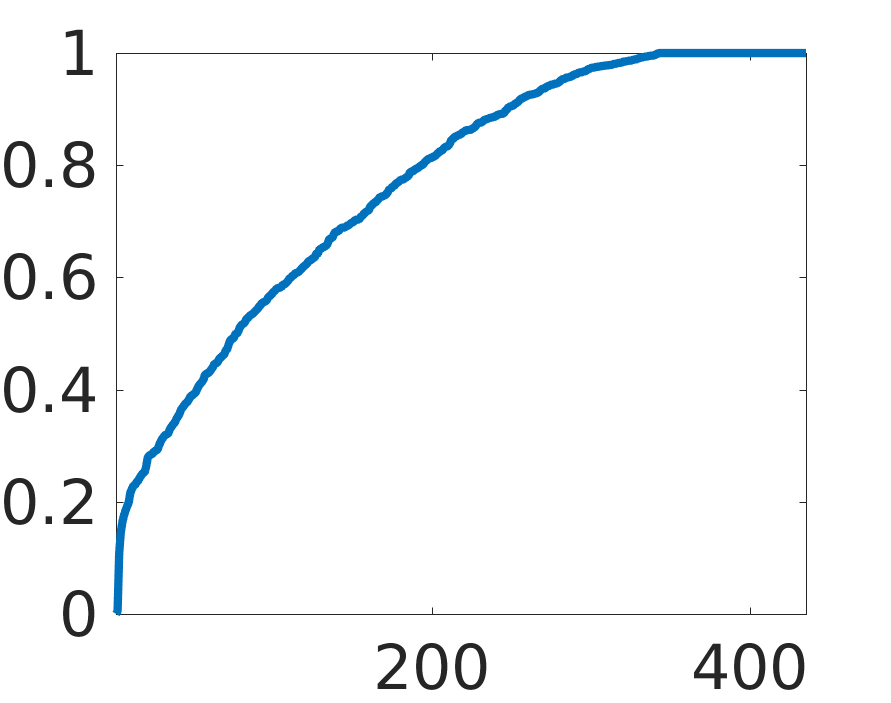}}
\caption{Spectra of graph Laplacian of various datasets. See Sec.\ref{sec:5} for the description of the datsets and graph construction parameters. The $y-$axis are the eigenvalues and the $x-$axis the index of ordering}\label{fig:spectra-all}
\end{figure}

\begin{algorithm}[!ht]
\caption{pCN Algorithm With Spectral Approximation}
\label{alg:pcneigfill}
\begin{algorithmic}[1]
\STATE{Input: $L$. $\Phi(u)$. $u^{(0)} \in U$.}
\STATE{Output: $M$ Approximate samples from the posterior distribution}
\STATE{Define: $\alpha(u,w) = \min\{1, \exp(\Phi(u) - \Phi(w)\}$.} 
\WHILE{$k < M$ }
\STATE{ $w^{(k)} = \sqrt{1 - \beta^2}u^{(k)} + \beta \xi^{(k)}_{\ell,o}$, where $\xi^{(k)}_{\ell, o} \sim \cN(0, C_{\ell,o})$ via equation (\ref{eq:Eigfill-Samples}). }
\STATE{Calculate acceptance probability $\alpha(u^{(k)}, w^{(k)})$.}
\STATE{Accept $w^{(k)}$ as $u^{(k+1)}$ with probability $\alpha(u^{(k)}, w^{(k)})$, otherwise $u^{(k+1)} = u^{(k)}$. }
\ENDWHILE
\end{algorithmic}
\end{algorithm}

\subsection{MAP Estimation: Optimization}
Recall that the objective function for the MAP estimation has the form 
$\frac12 \langle u, Pu \rangle + \Phi(u)$, where $u$ is supported on the 
space $U$. For Ginzburg-Landau and probit, the function $\Phi$ is smooth, 
and we can use a standard projected gradient method for the optimization. 
Since $L$ is typically ill-conditioned, it is preferable to use a semi-implicit
discretization as suggested in \cite{bertozzi2012diffuse}, as convergence to 
a stationary point can be shown under a graph independent learning rate. 
Furthermore, the discretization can be performed in terms of the eigenbasis 
$\{q_1, \dots, q_{N-1}\}$, which allows us to easily apply spectral projection 
when only a truncated set of eigenvectors is available. We state the algorithm 
in terms of the (possibly truncated) eigenbasis below. 
Here $P_{\ell}$ is an approximation to $P$ found by setting
$P_{\ell}=Q_{\ell} D_{\ell}  Q_{\ell}^*$ where
$Q_{\ell}$ is the matrix with columns $\{q_1, \cdots, q_{\ell-1}\}$
and $D_{\ell} = {\rm diag}(d)$ for $d(j)=c_{\ell}\lambda_j$, 
$j=1,\cdots, \ell-1.$ Thus $P_{N-1}=P.$

\begin{algorithm}[!ht]
\caption{Linearly-Implicit Gradient Flow with Spectral Projection}
\label{alg:spectralproj}
\begin{algorithmic}[1]
\STATE{\textbf{Input}: $Q_m = (q_1, \dots q_m)$, $\Lambda_m = (\lambda_1, \dots, \lambda_m)$,  $\Phi(u)$}, $u^{(0)} \in U$. 
\WHILE{$k < M$ }
\STATE{ $u^{(\star)} = u^{(k)} - \beta \nabla \Phi(u^{(k)})$}
\STATE{ $u^{(k+1)} = (I+\beta P_{m})^{-1}u^{(\star)}$}
\ENDWHILE
\end{algorithmic}
\end{algorithm}

\section{Numerical Experiments}
\label{sec:5}

%
In this section we conduct a series of numerical experiments on 
four different data sets that are representative of the field of
graph semi-supervised learning. There are three main purposes for the 
experiments. First we perform uncertainty quantification, as 
explained in subsection \ref{ssec:uq}. 
Secondly, we study the spectral approximation and projection variants on
pCN sampling as these scale well to massive graphs.
Finally we make some observations about the cost and practical
implementation details of these methods, for
the different Bayesian models we adopt; these
will help guide the reader in making choices about which algorithm to use.  We present the results for MAP estimation in Section B of the Appendix, alongside the proof of convexity of the
probit MAP estimator. 

The quality of the graph constructed from the feature vectors
is central to the performance of any graph learning algorithms. In the experiments below, we follow the graph construction procedures 
used in the previous papers 
\cite{bertozzi2012diffuse, hu2015multi, merkurjev2013mbo} which applied graph semi-supervised learning to all of the datasets that we consider in this paper.
Moreover, we have verified that for all the reported experiments below, 
the graph parameters are in a range such that 
spectral clustering \cite{von2007tutorial} (an unsupervised learning method)
gives a reasonable performance. The methods we employ
lead to refinements over spectral clustering (improved classification) and, 
of course, to uncertainty quantification (which spectral clustering
does not address).

\subsection{Data Sets}
\label{subsec:datasets}
We introduce the data sets and  describe the graph construction for each data 
set. In all cases we numerically construct the weight matrix $A$, and then
the graph Laplacian $L$.\footnote{The weight matrix $A$ is symmetric in 
theory; in practice we find that symmetrizing via the 
map $A \mapsto \frac{1}{2}A + \frac{1}{2}A^*$ is helpful.}

\subsubsection{Two Moons}
\label{subsubsec:twomoons}
The two moons artificial data set is constructed to give noisy data which
lies near a nonlinear low dimensional manifold embedded in a high dimensional
space \cite{buhler2009spectral}. The data set is constructed by
sampling $N$ data points uniformly from two semi-circles centered 
at $(0, 0)$ and $(1, 0.5)$ with radius $1$, 
embedding the data in $\mathbb{R}^d$, and adding Gaussian noise with 
standard deviation $\sigma.$ We set $N = 2,000$ and $d = 100$ in this paper;
recall that the graph size is $N$ and each feature vector has length $d$.
We will conduct a variety of experiments with different labelled
data size $J$, and in particular study variation with $J$.
The default value, when not varied, is $J$ at $3\%$ of $N$, with the
labelled points chosen at random.

We take each data point as a node on the graph, and construct a fully 
connected graph using the self-tuning weights of 
Zelnik-Manor and Perona \cite{zelnik2004self}, with $K$ = 10. 
Specifically we let $x_i$, $x_j$ be the coordinates of the data points 
$i$ and $j$. Then weight $a_{ij}$ between nodes $i$ and $j$ is defined by 
\begin{equation}
\label{eq:wz}
a_{ij} = \exp \Bigl(-\frac{\|x_i - x_j\|^2}{2 \tau_i \tau_j}\Bigr), 
\end{equation}
where $\tau_j$ is the distance of the $K$-th closest point to the node $j$. 

\subsubsection{House Voting Records from 1984}
This dataset contains the voting records of 435 U.S. House of Representatives;
for details see \cite{bertozzi2012diffuse} and the references therein.
The votes were recorded in 1984 from the $98^{th}$ United States Congress, 
$2^{nd}$ session. The votes for each individual is vectorized by mapping a yes 
vote to $1$, a no vote to $-1$, and an abstention/no-show to $0$. The
data set contains $16$ votes that are believed to be well-correlated 
with partisanship, and we use only these votes as feature 
vectors for constructing the graph. Thus the graph size is $N = 435$, 
and feature vectors have length $d = 16$. 
The goal is to predict the party affiliation of each individual, given a
small number of known affiliations (labels). We pick 
$3$ Democrats and $2$ Republicans at random to use as the observed class 
labels; thus $J=5$ corresponding to less than $1.2\%$ of fidelity 
(i.e. labelled) points.  
We construct a fully connected graph with weights given by
\eqref{eq:wz} with $\tau_j = \tau= 1.25$ for all nodes $j.$ 

\subsubsection{MNIST}
The MNIST database consists of
$70,000$ images of size $28 \times 28$ pixels 
containing the handwritten digits $0$ through $9$; see
\cite{lecun1998mnist} for details. Since in this paper we focus on 
binary classification, we only consider pairs of digits. To speed up 
calculations, we subsample randomly $2,000$ images from each digit to form a graph 
with $N = 4,000$ nodes; we use this for all our experiments except in
subsection \ref{ssec:CM} where we use the full data set of size 
$N={\cal O}(10^4)$ for digit pair $(4,9)$ to benchmark computational cost. 
The nodes of the graph are the images and as feature vectors we project the images onto the leading 50 principal components given by PCA;
thus the feature vectors at each node have length $d=50.$
We construct a $K$-nearest neighbor graph with $K = 20$ for each pair of digits considered. Namely, the weights $a_{ij}$ are non-zero if and only if one of $i$ or $j$ is in the $K$ nearest neighbors of the other.  The non-zero weights are set using \eqref{eq:wz} with $K=20$.

We choose the 
four pairs $(5,7)$, $(0, 6)$, $(3, 8)$ and  $(4,9)$. These four
pairs exhibit increasing levels of difficulty for classification.
This fact is demonstrated in Figures \ref{fig:mnist1} - \ref{fig:mnist4}, where we visualize the datasets by projecting the dataset onto the second and third eigenvector of the graph Laplacian. Namely, each node $i$ is mapped to the point $(Q(2, i), Q(3, i))\in \mathbb{R}^2$, where $L = Q \Lambda Q^*$.

\begin{figure}[!ht]
\centering
\subfloat[(4, 9)]{\label{fig:mnist1}\includegraphics[width=33mm]{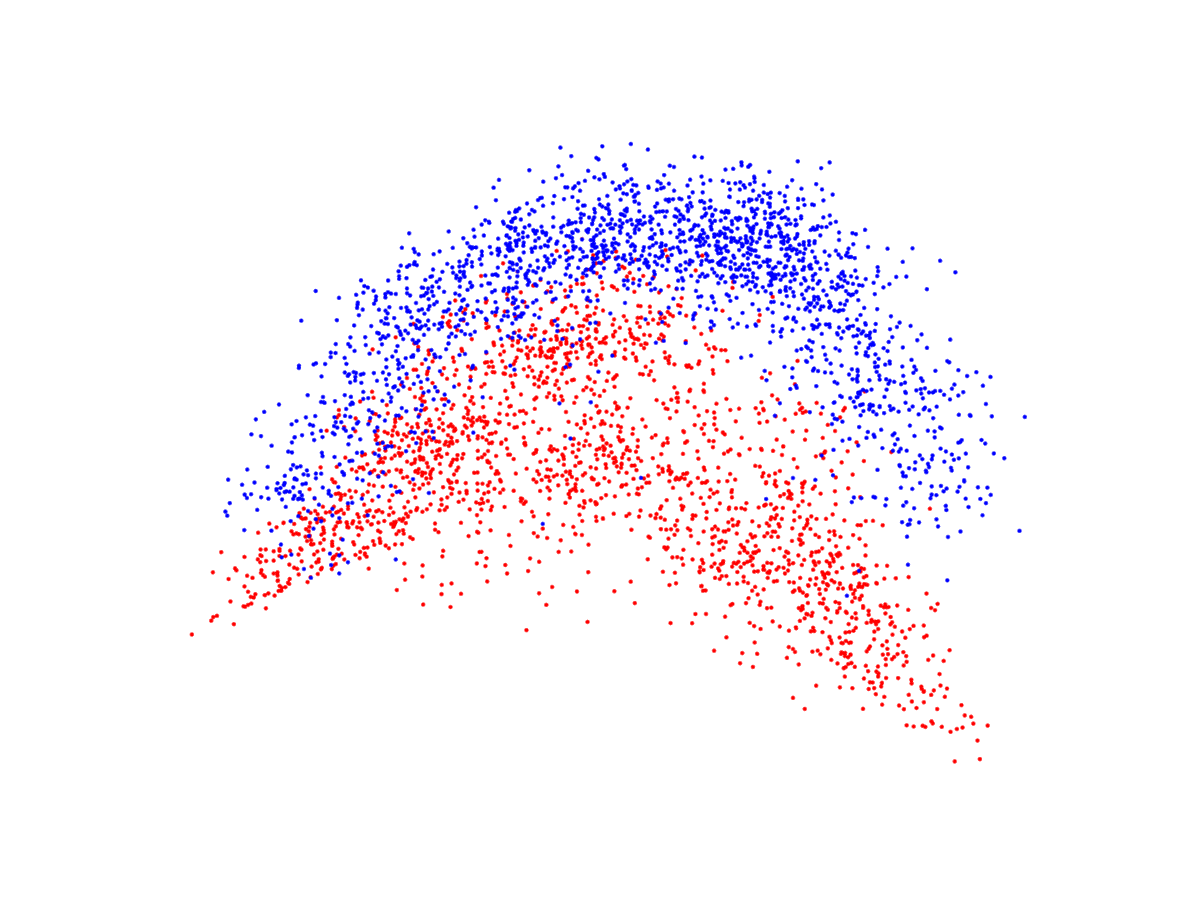}}
\subfloat[(3, 8)]{\label{fig:mnist2}\includegraphics[width=33mm]{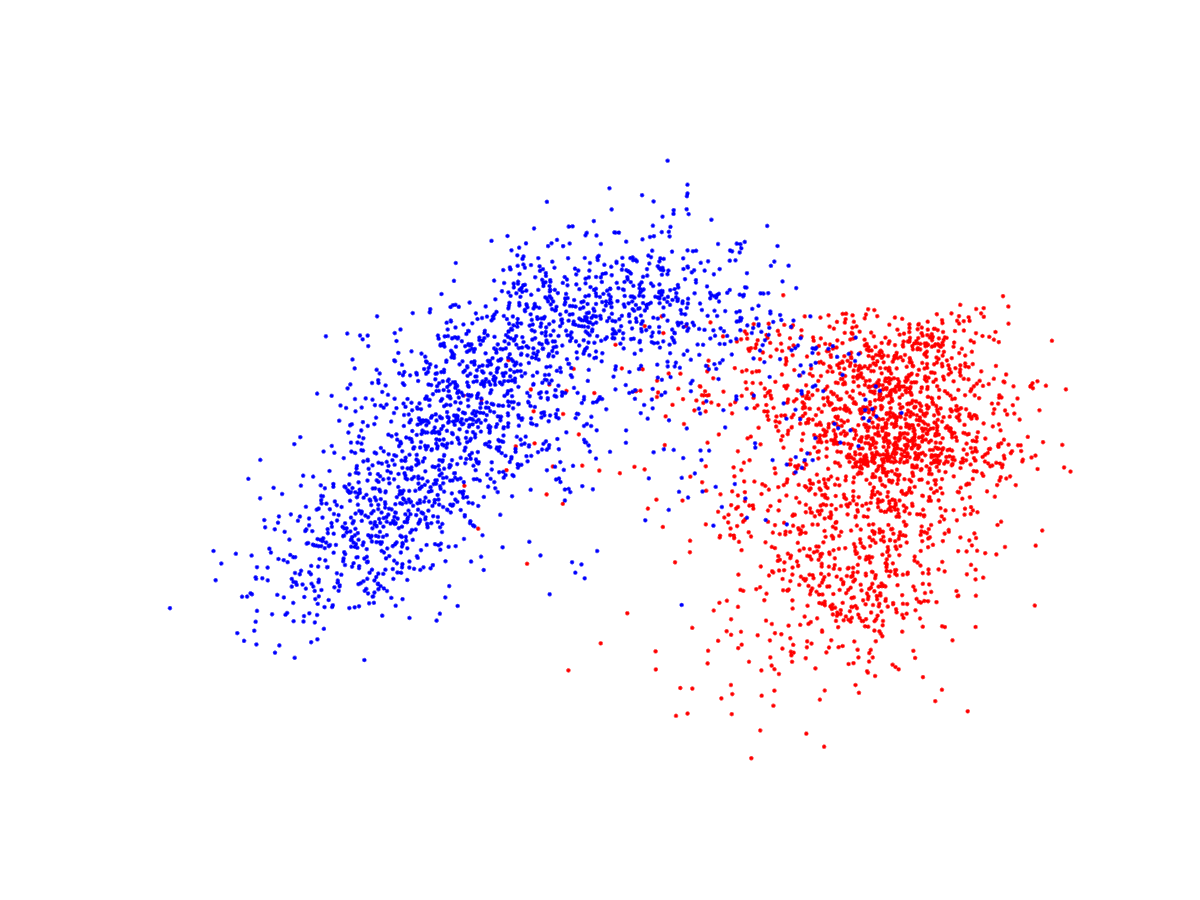}}
\subfloat[(0, 6)]{\label{fig:mnist3}\includegraphics[width=33mm]{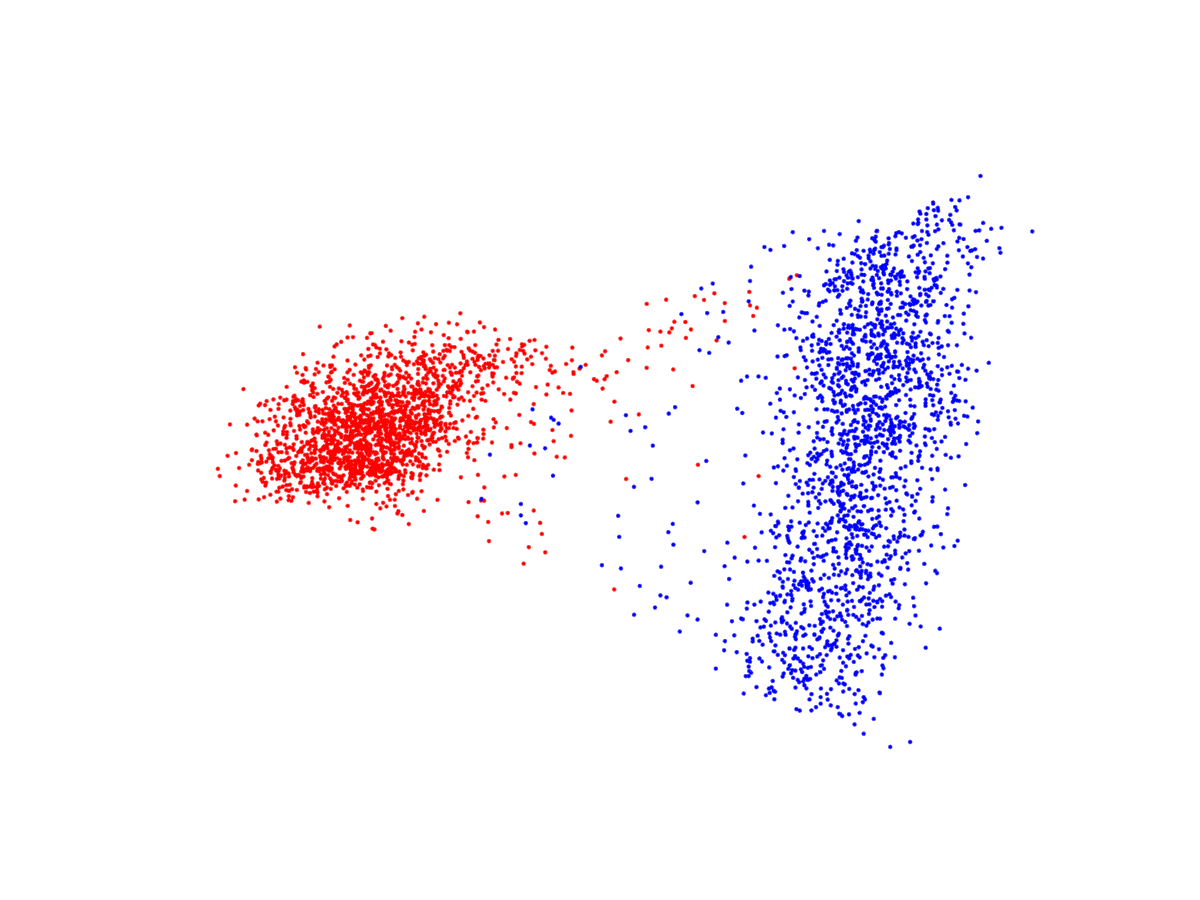}}
\subfloat[(5, 7)]{\label{fig:mnist4}\includegraphics[width=33mm]{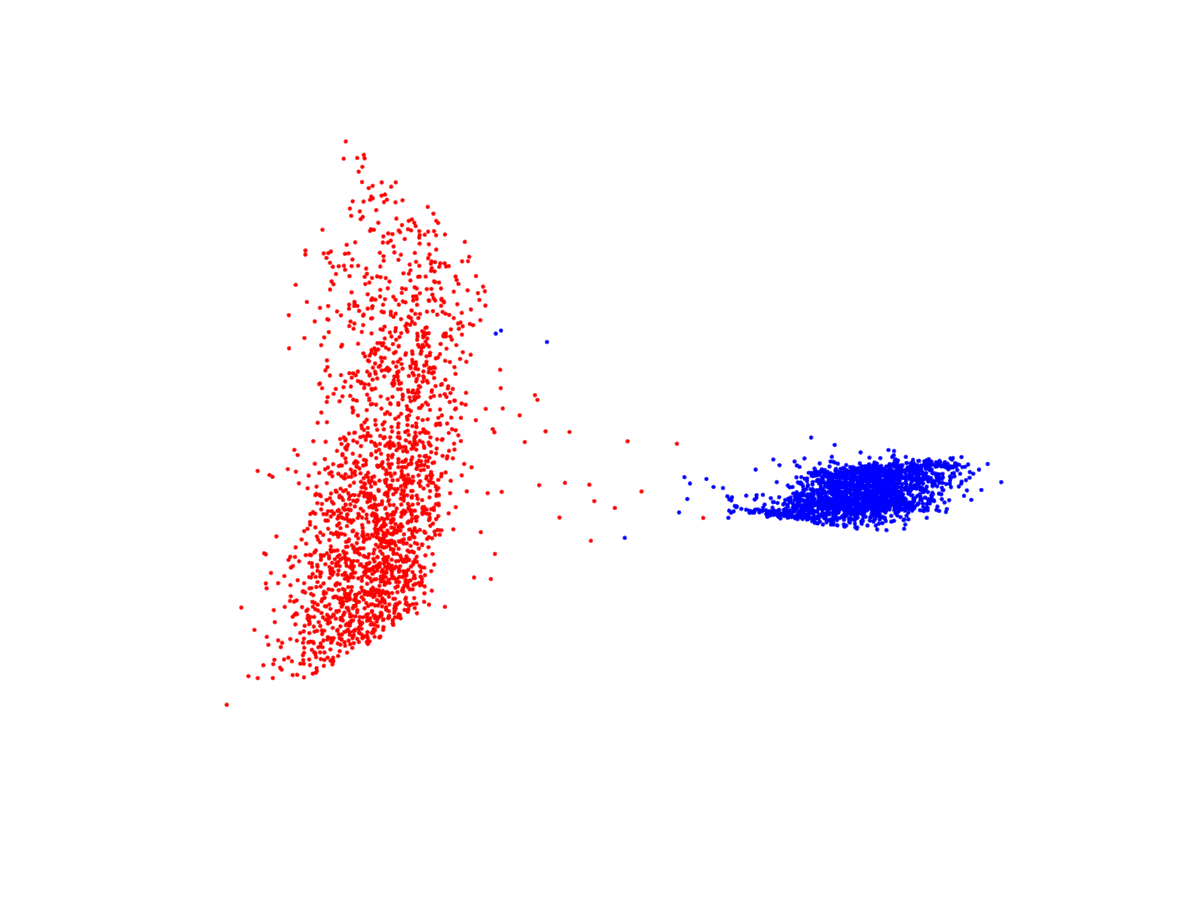}}
\caption{Visualization of data by projection onto $2^{nd}$ and $3^{rd}$ eigenfuctions of the graph Laplacian for the MNIST data set, where the vertical dimension is the $3^{rd}$ eigenvector and the horizontal dimension the $2^{nd}$. Each subfigure represents a different pair of digits. We construct a 20 nearest neighbour graph under the Zelnik-Manor and Perona scaling \cite{zelnik2004self} as in \eqref{eq:wz} with $K = 20$. }
\end{figure}

\subsubsection{HyperSpectral Image}
The hyperspectral data set analysed for this project was provided by the Applied Physics Laboratory at Johns Hopkins University; see
\cite{broadwater2011primer} for details.
It consists of a series of video sequences recording the release of 
chemical plumes taken at the Dugway Proving Ground. 
Each layer in the spectral dimension depicts a particular frequency starting 
at $7,830$ nm and ending with $11,700$ nm, with a channel spacing of $30$ nm,
giving $129$ channels; thus the feature vector has length $d = 129.$ 
The spatial dimension of each frame is $128 \times 320$ pixels.
We select $7$ frames from the video sequence as the input data, and consider each 
spatial pixel as a node on the graph. Thus the graph size
is $N = 128 \times 320 \times 7=286,720$. Note that time-ordering of
the data is ignored. The classification problem is to classify pixels 
that represent the chemical plumes against pixels that are the background. 

We construct a fully connected graph with weights given by the cosine distance:
$$w_{ij} = \frac{\langle x_i, x_j \rangle}{\|x_i\|\|x_j\|}.$$ 
This distance is small for vectors that point in the same direction,
and is insensitive to their magnitude. We consider the 
normalized Laplacian defined in \eqref{eq:s-lap}. Because it is
computationally prohibitive to compute eigenvectors of a Laplacian of this size, 
we apply the \nyst extension \cite{williams2000using,fowlkes2004spectral} 
to obtain an approximation to the true eigenvectors and eigenvalues;
see \cite{bertozzi2012diffuse} for details pertinent to the set-up here. 
We emphasize that each pixel in the $7$ frames is a node on the 
graph and that, in particular, pixels across the $7$ time-frames are also 
connected. Since we have no ground truth labels for this dataset, we generate known labels by setting the segmentation results from spectral clustering as ground truth. The default value of $J$ is $8,000$, and labels are chosen at random. This corresponds to labelling around
$2.8\%$ of the points. We only plot results for the last $3$ frames of the 
video sequence in order to ensure that the information in the figures it not
overwhelmingly large.

\subsection{Uncertainty Quantification}
In this subsection we demonstrate both
the feasibility, and value, of uncertainty quantification in
graph classification methods. We employ the probit and the Bayesian level-set model for most of the experiments in this subsection; we also employ
the Ginzburg-Landau model but since this can be slow to converge, 
due to the presence of local minima, it is only demonstrated on 
the voting records dataset. 
The pCN method is used for sampling on various datasets 
to demonstrate properties and interpretations
of the posterior. In all experiments, all statistics on the label $l$ 
are computed under the push-forward posterior measure onto label space,
$\nu$.

\subsubsection{Posterior Mean as Confidence Scores}

We construct the graph from the MNIST $(4,9)$ dataset following subsection 
\ref{subsec:datasets}. The noise variance $\gamma$ is set to $0.1$, 
and $4\%$ of fidelity points are chosen randomly from each class. The
probit posterior is used to compute \eqref{eq:posterior-mean}.
In Figure \ref{fig:mnistlowhigh} we demonstrate
that nodes with scores $s^l_j$ closer to the binary ground truth 
labels $\pm 1$ look visually more uniform than nodes with $s^l_j$ far from 
those labels. This shows that the posterior mean contains useful information
which differentiates between outliers and inliers that align with human 
perception.  The scores $s^l_j$ are computed as follows:  we let $\{u^{(k)}\}_{k=1}^M$ be a set of samples of the posterior measure obtained from the pCN algorithm. The probability $\bbP(S(u(j) = l(j))$ is approximated by 
$$\bbP\bigl(S(u(j) = l(j)\bigr) \approx \frac{1}{M}\sum_{k=1}^M\one_{u^{(k)}(j) > 0}$$ 
for each $j$.  Finally the score 
$$s^l_j = 2\bbP\bigl(S(u(j) = l(j)\bigr) - 1.$$

\begin{figure}[!ht]
\centering
\subfloat[Fours in MNIST]{\includegraphics[width=70mm]{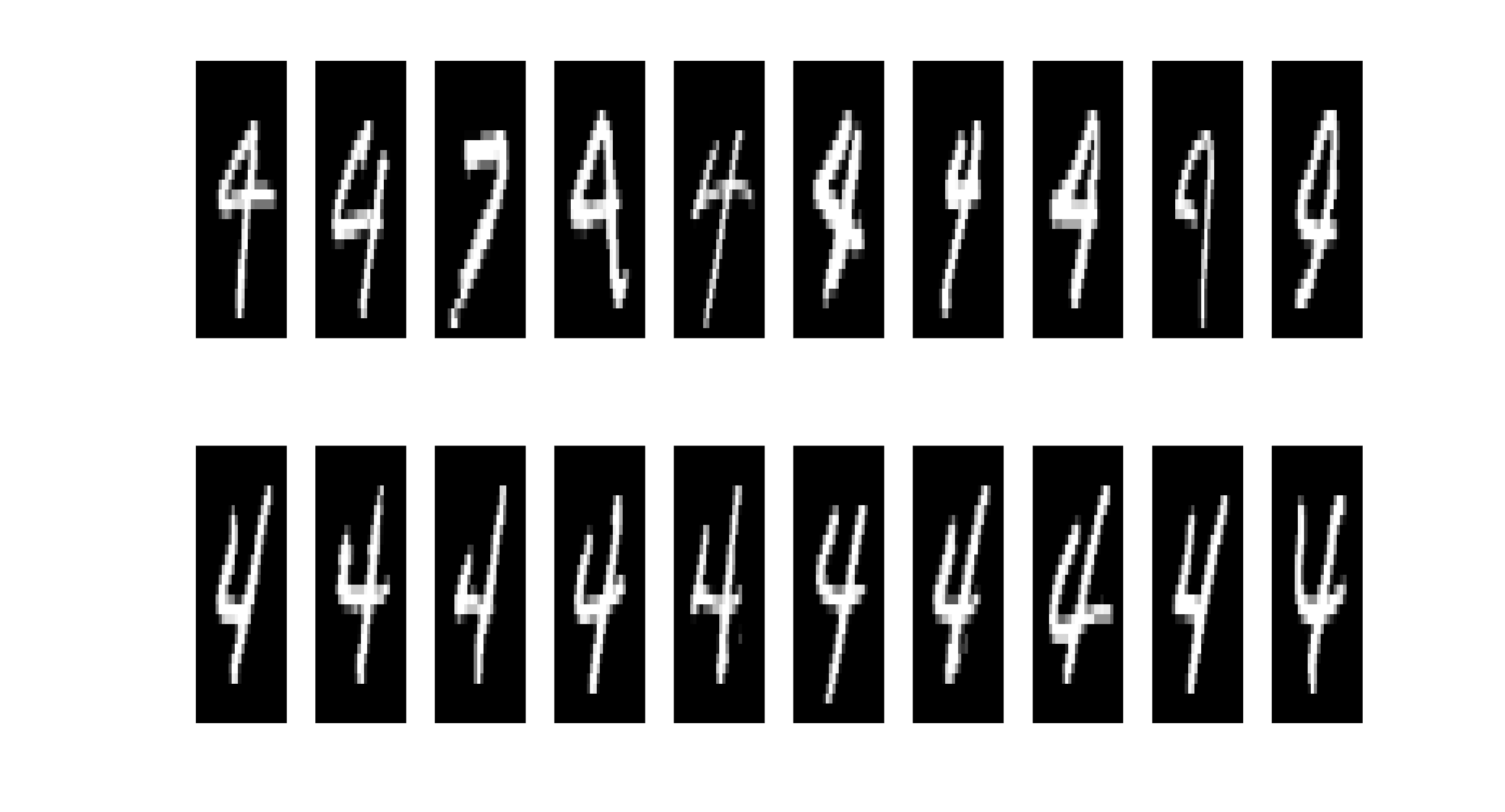}}
\subfloat[Nines in MNIST]{\includegraphics[width=70mm]{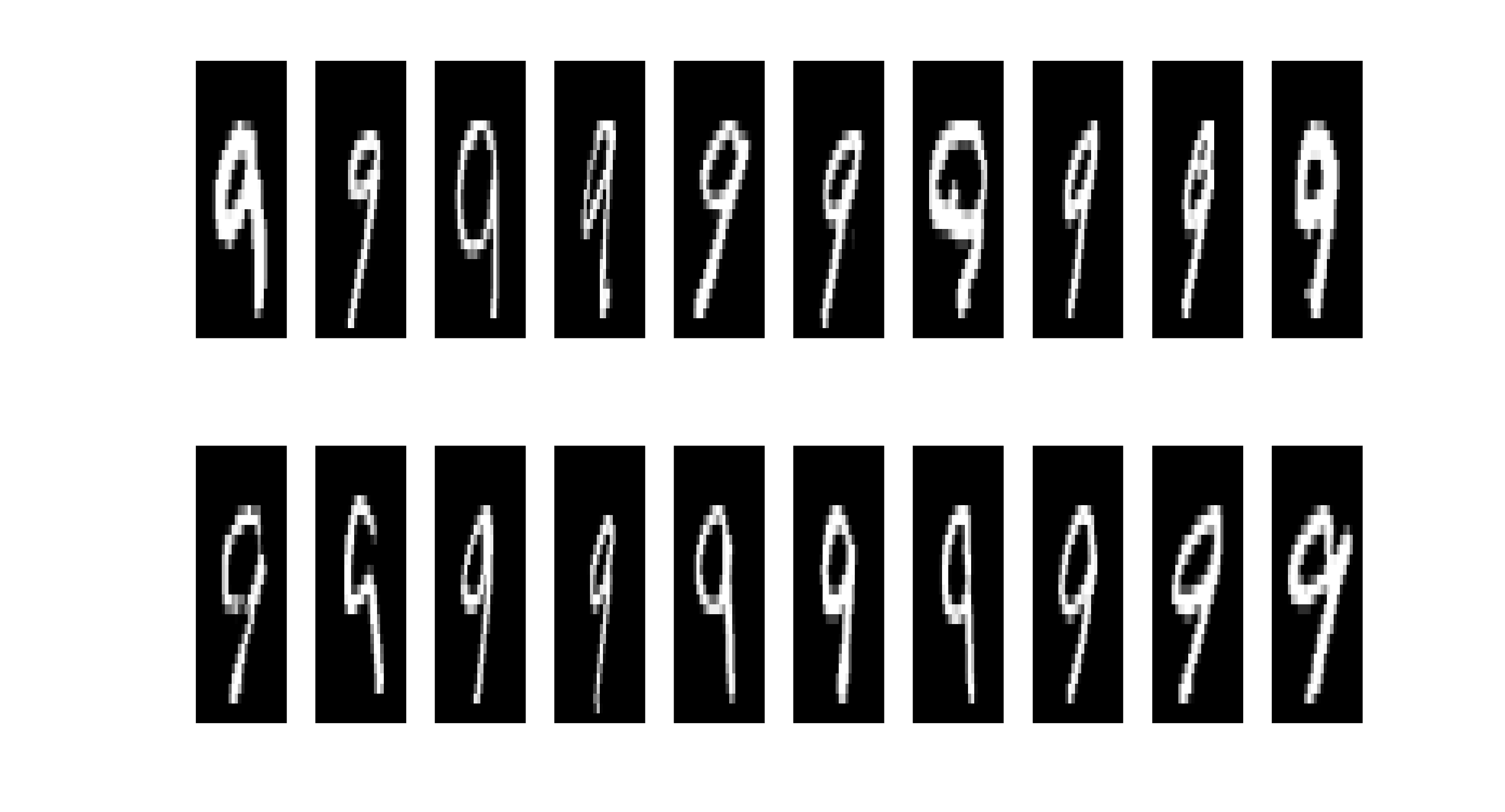}}
\caption{``Hard to classify" vs ``easy to classify" nodes in  the MNIST $(4, 9)$ dataset under the probit model. Here the digit ``4" is labeled +1 and ``9" is labeled -1. The top (bottom) row of the left column corresponds to images that have the lowest (highest) values of $s^l_j$ defined in \eqref{eq:posterior-mean} among images that have ground truth labels ``$4$". The right column is organized in the same way for images with ground truth labels $9$ except the top row now corresponds to the highest values of $s^l_j$. Higher $s^l_j$ indicates higher confidence that image $j$ is a $4$ and not a ``$9$", hence the top row could be interpreted as images that are ``hard to classify" by the current model, and vice versa for the bottom row. The graph is constructed as in Section \ref{sec:5}, and $\gamma = 0.1$, $\beta = 0.3$.  }\label{fig:mnistlowhigh}
\end{figure}

\subsubsection{Posterior Variance as Uncertainty Measure}
\label{subsubsec:feature-variance}
In this set of experiments, we show that the posterior distribution of the 
label variable $l = S(u)$ captures the uncertainty of the classification 
problem. We use the posterior variance of $l$, averaged over all nodes,
as a measure of the model variance; specifically formula \eqref{eq:mean-posterior-variance}.
We study the behaviour of this quantity as we vary the level of uncertainty 
within certain inputs to the problem. We demonstrate empirically that the 
posterior variance is approximately monotonic with respect to variations
in the levels of uncertainty in the input data, as it should be; and thus
that the posterior variance contains useful information about the
classification. 
We select quantities that reflect 
the separability of the classes in the feature space. 

Figure \ref{fig:twomoonsfeature} plots the posterior variance ${\rm Var}(l)$
against the standard deviation $\sigma$ of the noise appearing in the 
feature vectors for the two moons dataset; thus  points generated on
the two semi-circles overlap more as $\sigma$ increases. We employ
a sequence of posterior computations, using probit and Bayesian level-set, for  
$\sigma = 0.02:0.01:0.12$. 
Recall that $N=2,000$ and we choose $3\%$ of the nodes to have the ground truth 
labels as observed data. Within both models, $\gamma$ is fixed at $0.1$. 
A total of $1 \times 10^4$ samples are taken, and the proposal variance $\beta$ is set to $0.3$. We see that the mean posterior variance increases with $\sigma$, as is intuitively reasonable. Furthermore, because $\gamma$ is small,  
probit and Bayesian level-set are very similar models and this is reflected
in the similar quantitative values for uncertainty.

\begin{figure}[!ht]
\centering
\includegraphics[height=4.2cm, width=8cm]{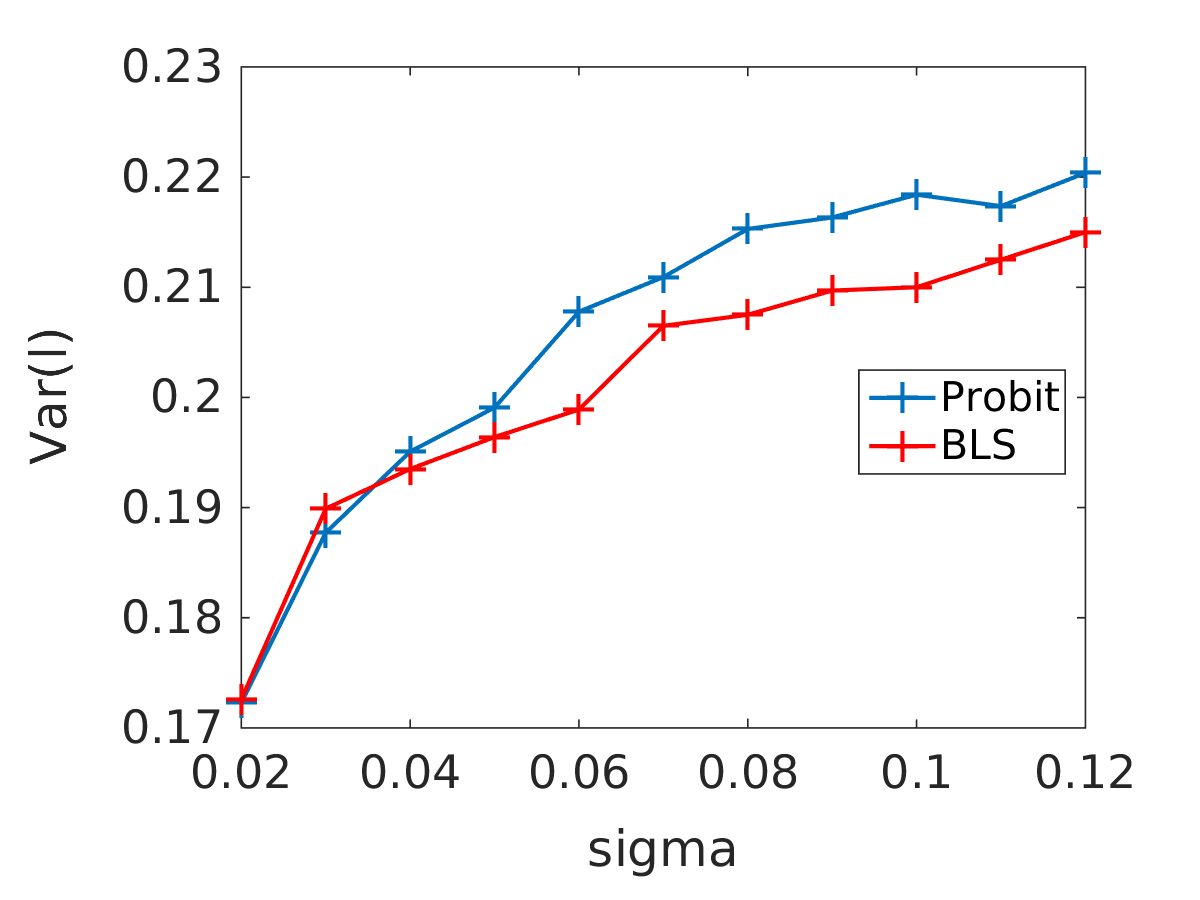}
\caption{Mean Posterior Variance defined in \eqref{eq:mean-posterior-variance} versus feature noise $\sigma$ for the probit model and the BLS model applied to the Two Moons Dataset with $N = 2,000$. For each trial, a realization of the two moons dataset under the given parameter $\sigma$ is generated, where $\sigma$ is the Gaussian noise on the features defined in Section \ref{subsubsec:twomoons} , and $3\%$ of nodes are randomly chosen as fidelity. We run $20$ trials for each value of $\sigma$, and average the mean posterior variance across the $20$ trials in the figure. We set  $\gamma = 0.1$ and $\beta = 0.3$ for both models. } \label{fig:twomoonsfeature}
\end{figure}

A similar experiment studies the posterior label variance ${\rm Var}(l)$ 
as a function of the pair of digits classified within the MNIST data set.
We choose $4\%$ of the nodes as labelled data, and set $\gamma = 0.1$. The number of samples employed
is $1 \times 10^4$ and the proposal variance $\beta$ is set to be $0.3$.
Table \ref{tab:mnistpostvar} shows the posterior label variance.
Recall that Figures \ref{fig:mnist1} - \ref{fig:mnist4} suggest that
the pairs $(4, 9), (3, 8), (0, 6), (5, 7)$ are increasingly easy to separate, 
and this is reflected in the decrease of the posterior label variance
shown in Table \ref{tab:mnistpostvar}.

\begin {table}[h!]
\centering
\begin{tabular}{|l|l|l|l|l|} \hline 
Digits 	 &	(4, 9) & (3, 8) & (0, 6)&(5, 7) \\ \hline
probit &  0.1485 & 0.1005 & 0.0429 &0.0084\\ \hline
BLS &  0.1280 & 0.1018 & 0.0489 &0.0121\\ \hline
\end{tabular}  
 \label{tab:mnistpostvar}  \caption {Mean Posterior Variance of different digit pairs for the probit model and the BLS model applied to the MNIST Dataset. The pairs are organized from left to right according to the separability of the two classes as shown in Fig.\ref{fig:mnist1} - \ref{fig:mnist4}. For each trial, we randomly select $4\%$ of nodes as fidelity. We run $10$ trials for each pairs of digits and average the mean posterior variance across trials. We set $\gamma = 0.1$ and $\beta = 0.3$ 
for both models.  }
\end{table}

The previous two experiments in this subsection have studied posterior
label variance ${\rm Var}(l)$ as a function of variation in the prior 
data.
We now turn and study how posterior variance changes as a function of varying
the likelihood information, again for both two moons and MNIST data sets. 
In Figures \ref{fig:twomoonslabelnoise1} and \ref{fig:mnistlabelnoise1}, we
plot the posterior label variance against the percentage of nodes observed. 
We observe that the observational variance decreases as the amount  
of labelled data increases. 
Figures \ref{fig:twomoonslabelnoise2} and \ref{fig:mnistlabelnoise2} show that 
the posterior label variance increases almost monotonically as observational noise $\gamma$ increases. 
Furthermore the level set and probit formulations produce similar
answers for $\gamma$ small, reflecting the close similarity between 
those methods when $\gamma$ is small -- when $\gamma=0$
their likelihoods coincide.

In summary of this subsection, the label posterior variance ${\rm Var}(l)$
behaves intuitively as expected as a function of varying the prior and
likelihood information that specify the statistical probit model 
and the Bayesian level-set model.
The uncertainty quantification thus provides useful, and consistent,
information that can be used to inform decisions made on the basis of
classifications.

\begin{figure}[!ht]
\centering
\subfloat[Two Moons]{\label{fig:twomoonslabelnoise1}\includegraphics[height=37mm, width=65mm]{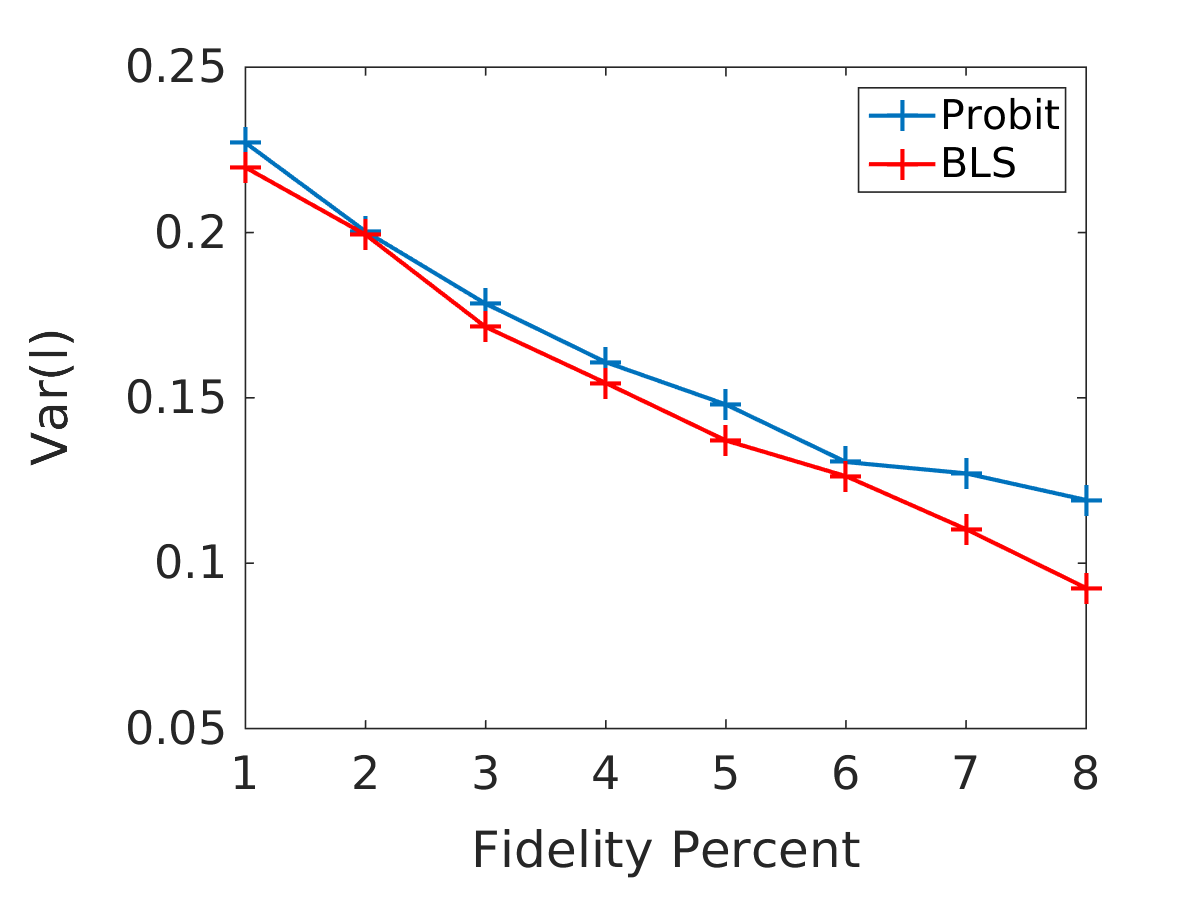}}
\subfloat[MNIST49]{\label{fig:mnistlabelnoise1}\includegraphics[height=37mm, width=65mm]{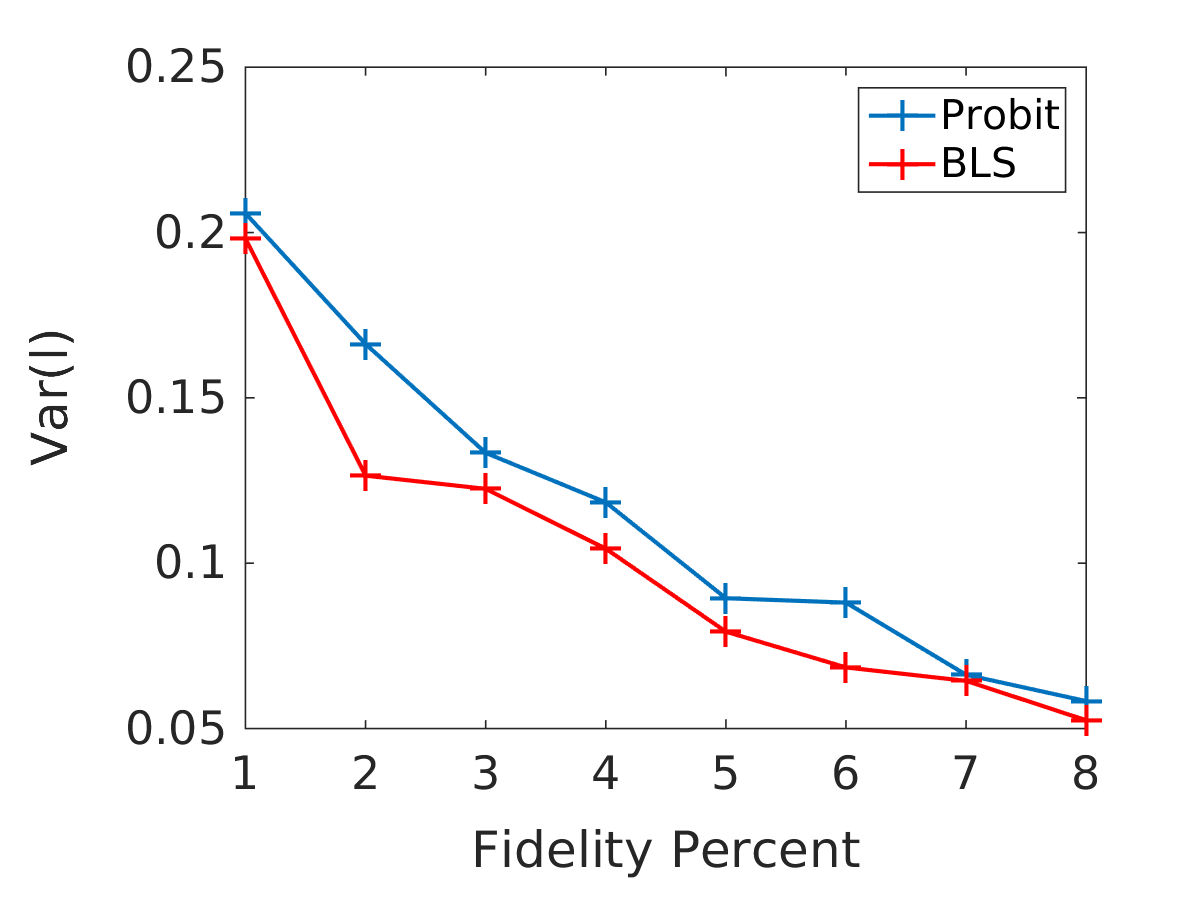}} \\
\subfloat[Two Moons]{\label{fig:twomoonslabelnoise2}\includegraphics[height=37mm, width=65mm]{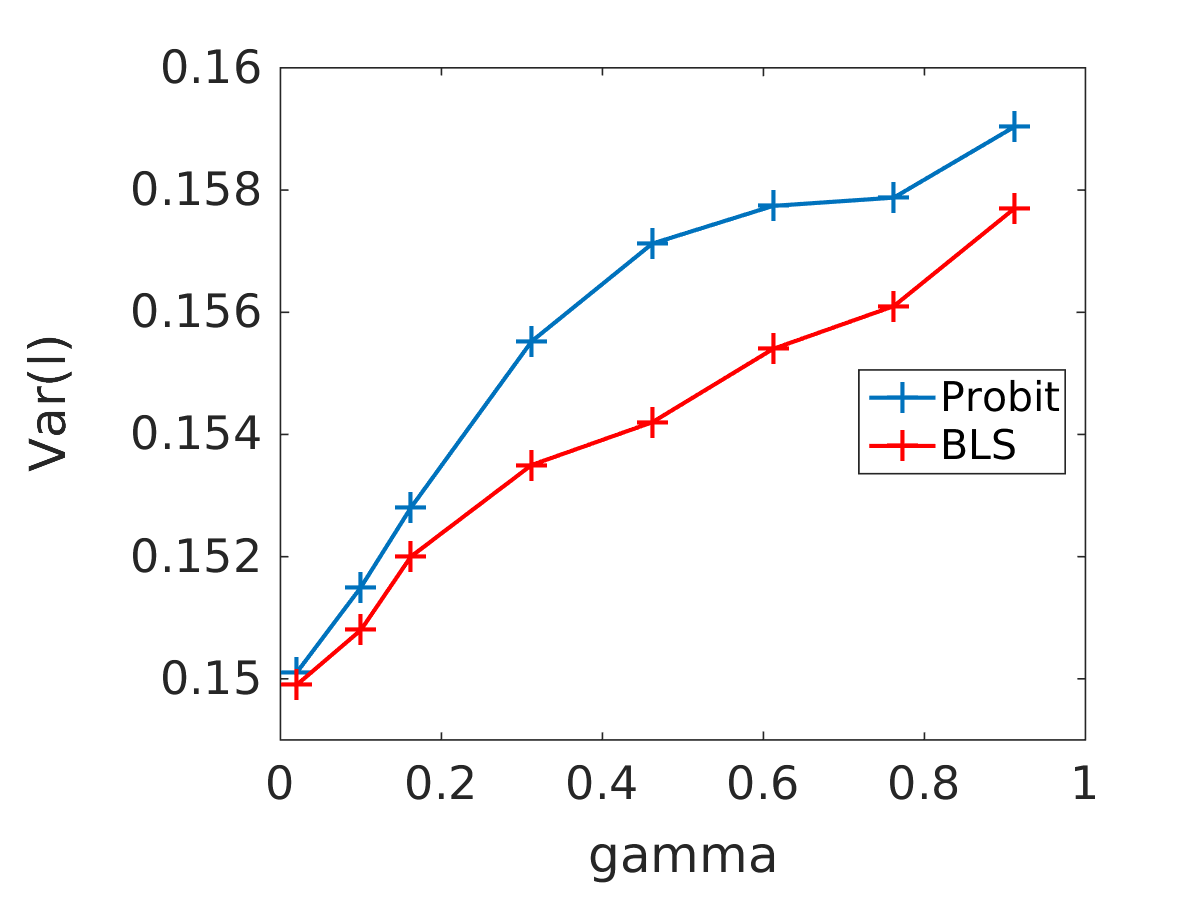}}
\subfloat[MNIST]{\label{fig:mnistlabelnoise2}\includegraphics[height=37mm, width=65mm]{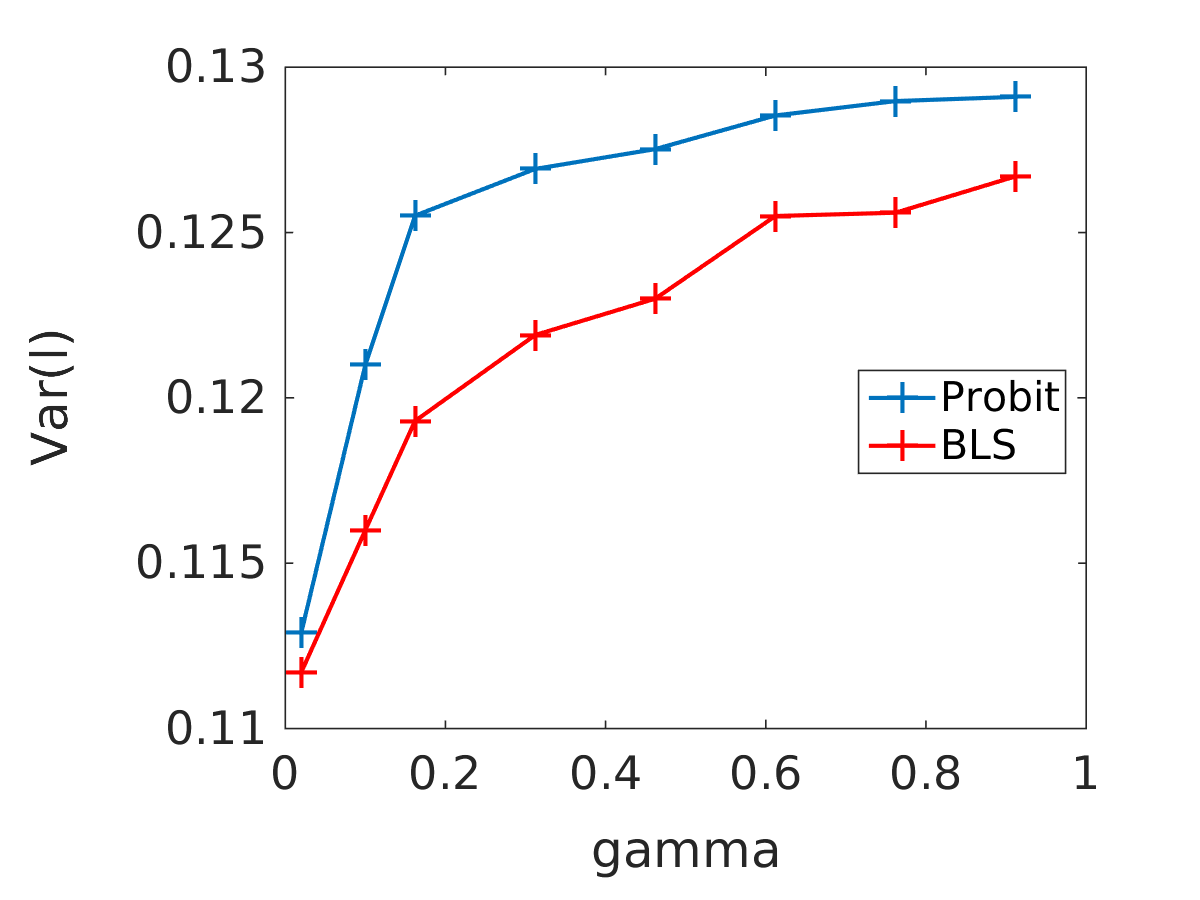}}
\caption{Mean Posterior Variance as in \eqref{eq:mean-posterior-variance} versus percentage of labelled points and noise level $\gamma$ for the probit model and the BLS model applied to the Two Moons dataset and the 4-9 MNIST dataset. For two moons, we fix $N = 2,000$ and $\sigma = 0.06$. For each trial, we generate a realization of the two moons dataset while the MNIST dataset is fixed. For a), b) $\gamma$ is fixed at $0.1$, and a certain percentage of nodes are selected at random as labelled. For c), d), the proportion of labelled points is
fixed at $4\%$, and $\gamma$ is varied across a range. Results are averaged over 20 trials. }
\end{figure}

\subsubsection{Visualization of Marginal Posterior Density}
In this subsection, we contrast the posterior distribution 
$\bbP(v|y)$ of the Ginzburg-Landau model with that of the probit and Bayesian level-set (BLS) models.
The graph is constructed from the voting records data with the fidelity points 
chosen as described in subsection \ref{subsec:datasets}. In  Figure
\ref{fig:glvsprobitmcmc} we plot the histograms of the empirical 
marginal posterior distribution on $\bbP(v(i)|y)$ and $\bbP(u(i)|y)$ 
for a selection of nodes on the graph.
For the top row of Figure \ref{fig:glvsprobitmcmc}, we select $6$ nodes with ``low confidence" predictions, and plot the empirical marginal distribution of $u$ for probit and BLS, and that of $v$ for the Ginzburg-Landau model. Note that the same set of nodes is chosen for different models. 
The plots in this row demonstrate the multi-modal nature of the 
Ginzburg-Landau distribution in contrast to the uni-modal nature of the 
probit posterior; this uni-modality is a consequence of the log-concavity of the probit likelihood. For the bottom row, we plot the same empirical distributions for 6 nodes with ``high confidence" predictions. In contrast with the top row, the Ginzburg-Landau marginal for high confidence nodes is essentially uni-modal since most samples of $v$ evaluated on these nodes have a fixed sign.



\begin{figure}[!ht]
\centering
\subfloat[Ginzburg-Landau (Low)]{\includegraphics[width=45mm]{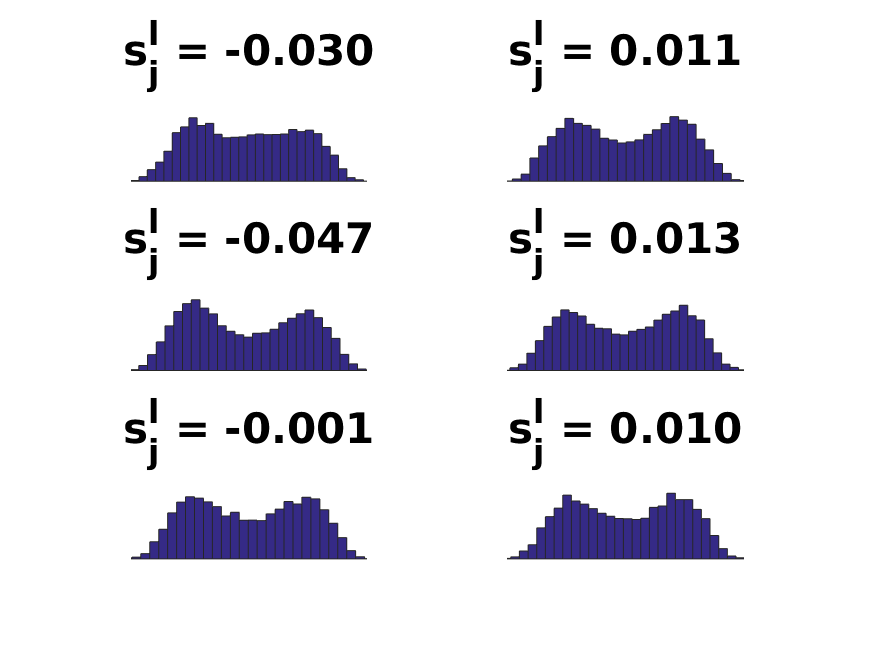}}
\subfloat[probit (Low)]{\includegraphics[width=45mm]{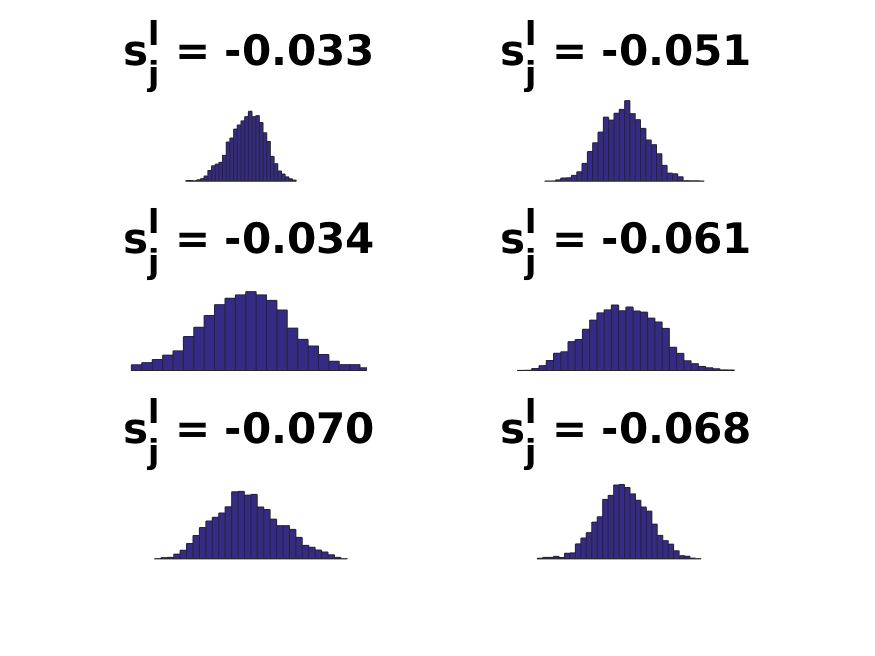}}
\subfloat[BLS (Low)]{\includegraphics[width=45mm]{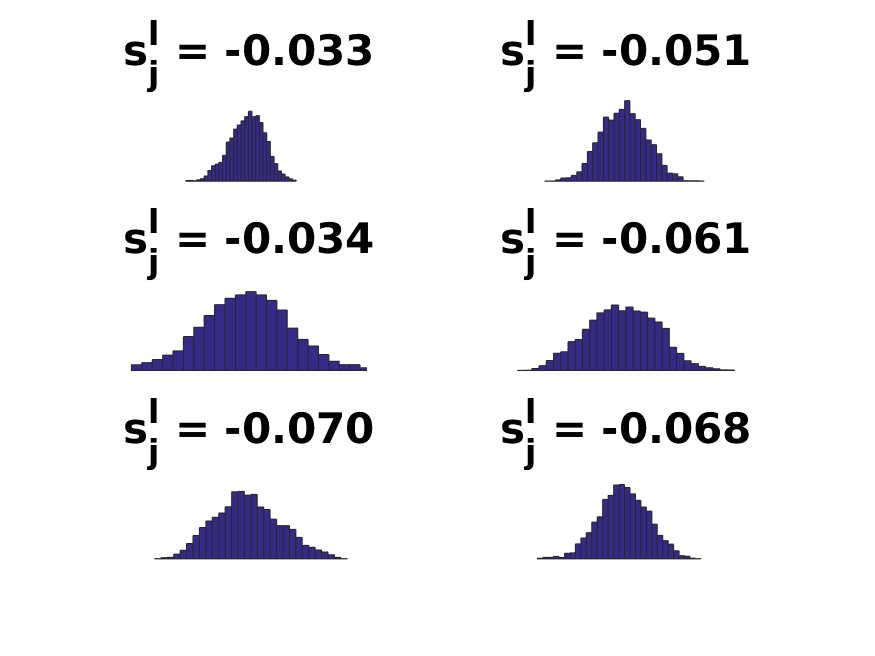}} \\
\subfloat[Ginzburg-Landau (High)]{\includegraphics[width=45mm]{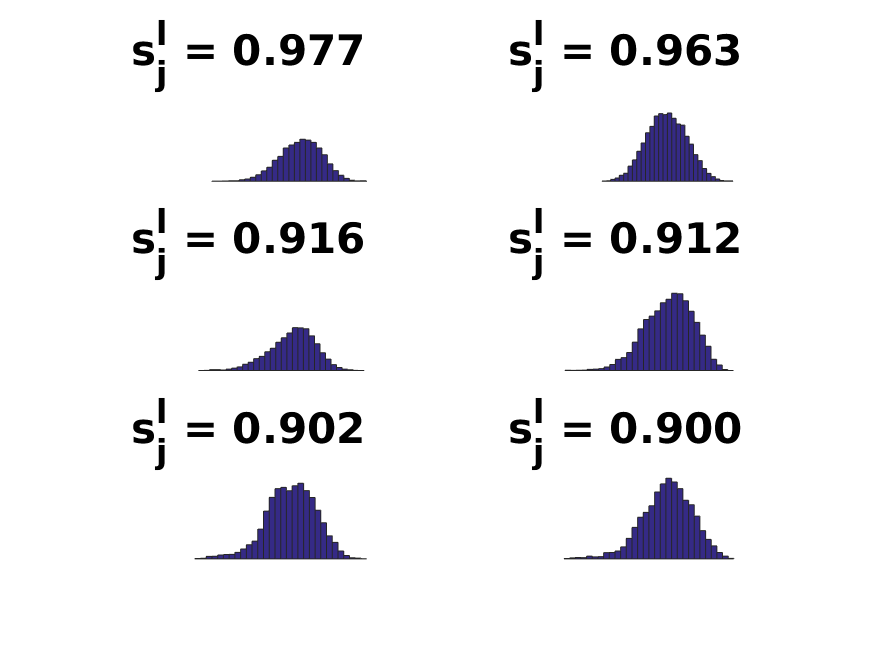}}
\subfloat[probit (High)]{\includegraphics[width=45mm]{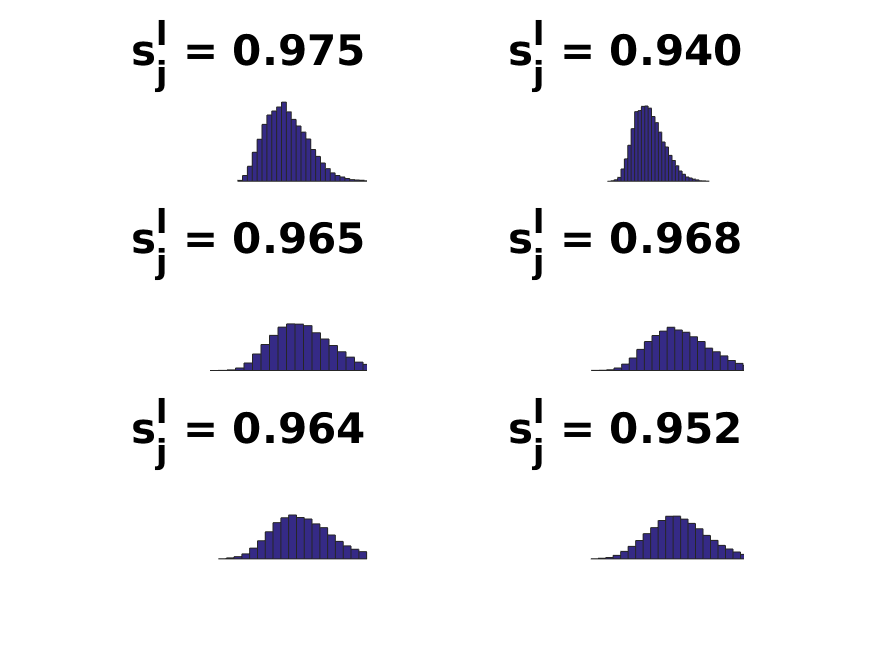}}
\subfloat[BLS (High)]{\includegraphics[width=45mm]{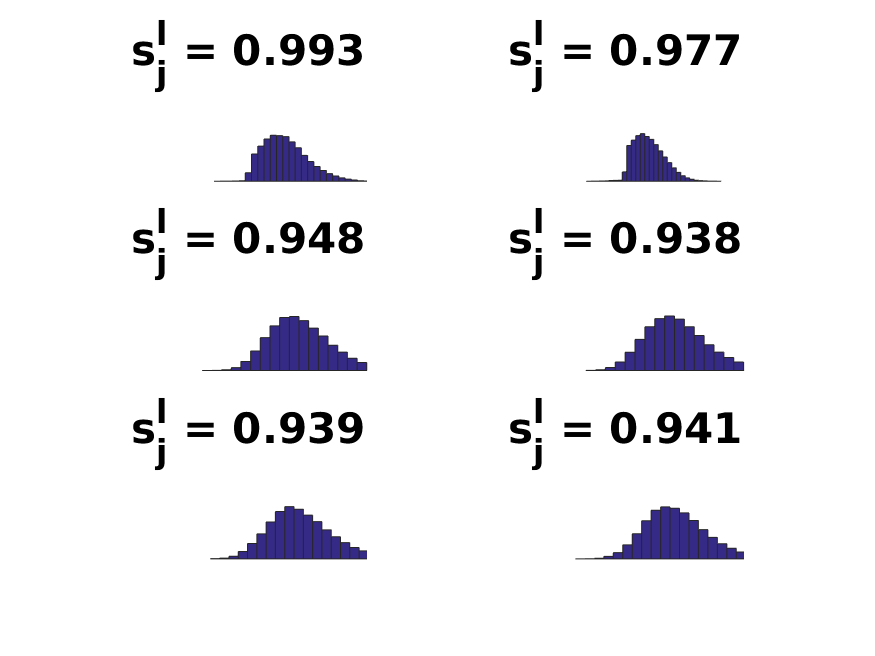}} 
\caption{Visualization of marginal posterior density for low and high confidence predictions across different models. 
Each image plots the empirical marginal posterior density 
of a certain node $i$, obtained from the histogram of $1\times 10^5$ approximate samples using pCN. Columns in the figure (e.g. a) and d)) are grouped by model. From left to right, the models are Ginzburg-Landau, probit, and Bayesian level-set respectively. From the top down, the rows in the figure (e.g. a)-c)) denote the low confidence and high confidence predictions respectively. For the top row, we select $6$ nodes with the lowest absolute value of the posterior mean $s^l_j$, defined in equation (\ref{eq:posterior-mean}), averaged across three models. For the bottom row, we select nodes with the highest average posterior mean $s^l_j$. We show the posterior mean $s^l_j$ on top of the histograms for reference. The experiment parameters are: $\epsilon = 10.0$, $\gamma = 0.6$, $\beta = 0.1$ for the Ginburg-Landau model, and $\gamma =  0.5$, $\beta = 0.2$ for the probit and BLS model.} \label{fig:glvsprobitmcmc}
\end{figure}

\subsection{Spectral Approximation and Projection Methods}
Here we discuss Algorithms \ref{alg:pcneig} and \ref{alg:pcneigfill},
designed to approximate the full (but expensive on large graphs)
Algorithm \ref{alg:pcnfull}.

First, we examine the quality of the approximation by applying the algorithms to the voting records dataset, a small enough problem where sampling using the full graph Laplacian is feasible. To quantify the quality of approximation, we compute the posterior mean of the thresholded variable $s^l_j$ for both Algorithm \ref{alg:pcneig} and Algorithm \ref{alg:pcneigfill}, and compare the mean absolute difference $\frac{1}{N}|s^l_j - s^{l*}_j|$ where $s^{l*}_j$ is the ``ground truth" value computed using the full Laplacian. Using $\gamma = 0.1$, $\beta = 0.3$, and a truncation level of $\ell = 150$, we observe that the mean absolute difference for spectral projection is \textbf{0.1577}, and \textbf{0.0261} for spectral approximation. In general, we set $\bar{\lambda}$ to be $\max_{j\leq \ell} \lambda_j$ where $\ell$ is the truncation level.

Next we apply the spectral projection/approximation algorithms with the Bayesian level-set likelihood to the hyperspectral image dataset;
the results for probit are similar (when we use small $\gamma$) but have
greater cost per step, because of the cdf evaluations required for probit.
The first two rows in Fig.\ref{fig:hsmcmcproj} show that the posterior mean $s^l_j$ is able to differentiate between 
different concentrations of the plume gas. We have also coloured pixels with 
$|s^l_j|<0.4$ in red to highlight the regions with greater levels of uncertainty. 
We observe that the red pixels mainly lie in the edges of the gas plume, which 
conforms with human intuition.  As in the voting records
example in the previous subsection, the spectral approximation method 
has greater posterior uncertainty, demonstrated by the greater number 
of red pixels in the second row of Fig.\ref{fig:hsmcmcproj} compared to the first row.
We conjecture that the spectral approximation is closer to what would be
obtained by sampling the full distribution, but we have not verified this
as the full problem is too large to readily sample. 
The bottom row of Fig.\ref{fig:hsmcmcproj} shows the result of using optimization
based classification, using the Gibzburg-Landau method. This is shown simply to
demonstrate consistency with  the full UQ approach shown 
in the other two rows, in terms of hard classification.

\begin{figure}[!ht]
\centering
\includegraphics[width=130mm, height=60mm]{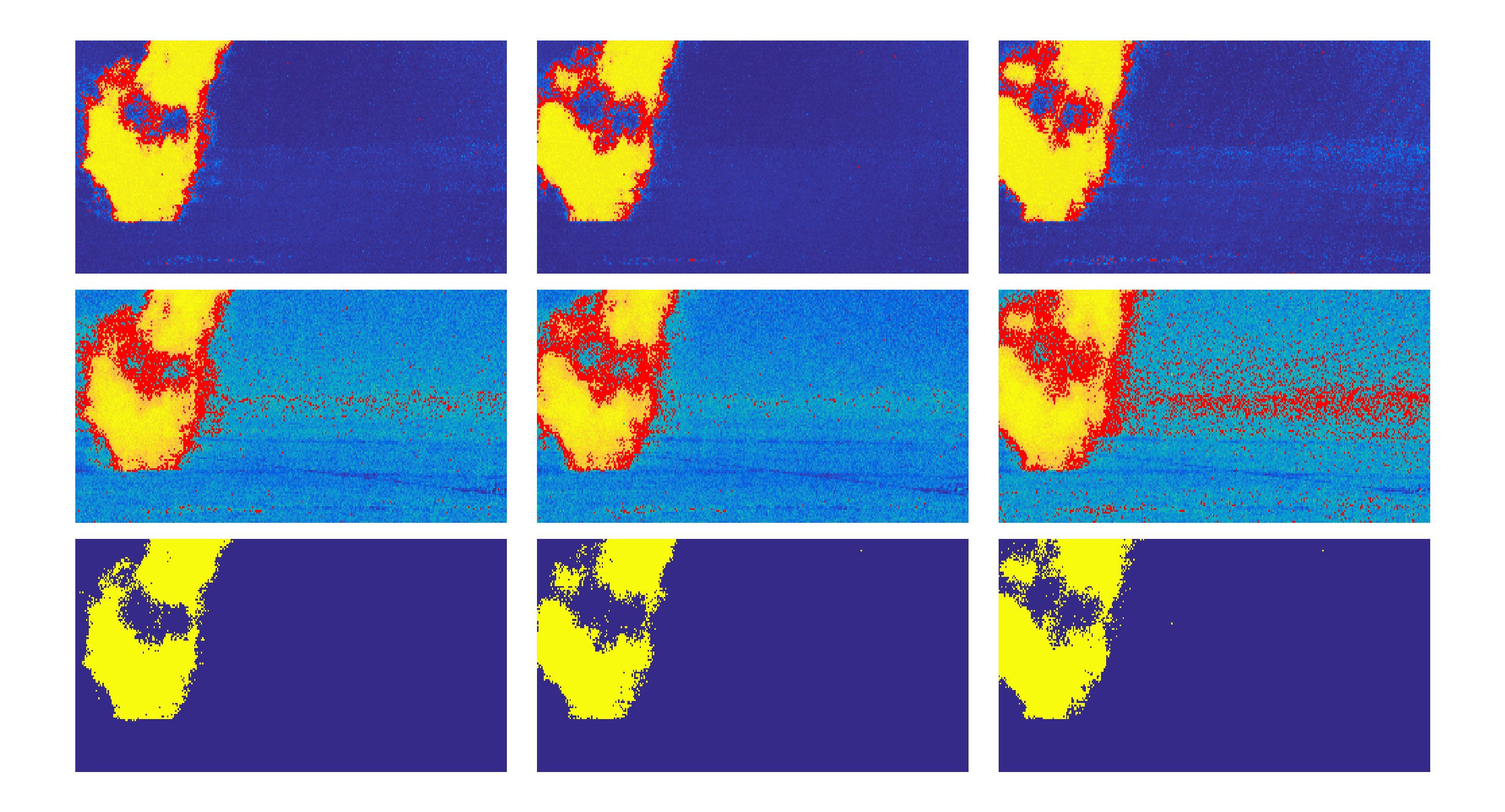}
\caption{Inference results on hyperspectral image dataset using spectral projection (top row), spectral approximation (middle row), and Ginzburg-Landau classification (bottom row). For the top two rows, the values of $s^l_j$ are plotted on a $[-1, 1]$ color scale on each pixel location. In addition, we highlight the regions of uncertain 
classification by coloring the pixels with $|s^l_j|<0.4$ in red.  The bottom row is the classification result from the Ginzburg-Landau model, shown here as a comparison.
The truncation level $\ell=40$, and for the spectral approximation algorithm, $\bar{\lambda} = 1$. 
We set $\gamma = 0.1$, $\beta=0.08$ and use $M=2 \times 10^4$ MCMC samples. 
We create the label data
by subsampling $8,000$ pixels ($\approx 2.8\%$ of the total) from the
labellings obtained by spectral clustering.}
\label{fig:hsmcmcproj}
\end{figure}

\subsection{Comparitive Remarks About The Different Models}
\label{ssec:CM}

At a high level we have shown the following concerning the three models
based on probit, level-set and Ginzburg-Landau:

\begin{itemize}

\item Bayesian
level set is considerably cheaper to implement than probit in Matlab because the
norm cdf evaluations required for probit are expensive. 

\item  Probit and Bayesian level-set behave similarly, for posterior
sampling, especially for
small $\gamma$, since they formally coincide when $\gamma=0.$ 

\item Probit and Bayesian level-set are superior to Ginzburg-Landau for
posterior sampling; this is because probit has log-concave posterior, 
whilst Ginzburg-Landau is multi-modal.

\item Ginzburg-Landau provides the best hard classifiers, when used
as an optimizer (MAP estimator), and provided it is initialized well. 
However it behaves poorly when not initialized
carefully because of multi-modal behaviour. In constrast probit provides almost
indistinguihsable classifiers, comparable or marginally worse in terms of accuracy,
and has a convex objective function and hence a unique minimizer.  (See Appendix for details of the relevant experiments.)

\end{itemize}

We expand on the details of these conclusions by studying run times of the
algorithms. All experiments are done on a 1.5GHz machine with Intel Core i7.
In Table \ref{tab:timeing}, we compare the running time of the MCMC for different models on various datasets. We use an a posteriori condition on the samples 
$u^{(k)}$ to empirically determine the sample size $M$ needed for the MCMC to converge. Note that this condition is by no means a replacement for a rigorous analysis of convergence using auto-correlation, but is designed to provide a ballpark estimate of the speed of these algorithms on real applications. We now define the a posteriori condition used. Let the approximate samples be $\{u^{(k)}\}$. We define the cumulative average as $\tilde{u}^{(k)} = \frac{1}{k} \sum_{j=1}^k u^{(j)}$, and find the first $k$ such that 
\begin{equation}
\|\tilde{u}^{(kT)} - \tilde{u}^{((k-1)T)} \| \leq {\tt tol}, 
\end{equation}
 where ${\tt tol}$ is the tolerance and $T$ is the number of iterations skipped. We set $T = 5000$, and also tune the stepsize parameter $\beta$ such that the average acceptance probability of the MCMC is over $50\%$. We choose the model parameters according to the experiments in the sections above so that the posterior mean gives a reasonable classification result.  
\begin{table}
\begin{tabular}{|l|l|l|l|} \hline 
Data 	 &	Voting Records & MNIST49 & Hyperspectral \\ 
(Tol) & ${\tt tol} = 1 \times 10^{-3}$ & ${\tt tol} = 1.5\times 10^{-3}$ & ${\tt tol} = 2\times 10^{-2}$ \\ 
(N) & $N = 435$ & $N\approx 1.1\times 10^4$ &  $N\approx 2.9\times 10^5$\\
(Neig) & $Neig = 435$ & $Neig = 300$ & $Neig = 50$\\ 
(J) & $J=5$ & $J=440$ & $J=8000$ \\ \hline
Preprocessing &   $t=0.7s$ & $t=50.8s$  & $t < 60s $\tablefootnote{According to the reporting in \cite{merkurjev2014graph}. } \\ \hline
probit &  $t=8.9s,$ & $t=176.4s, $ & $t = 5410.3 s,$\\ 
 &  $M=10^4$ & $ M=1.5\times 10^4$ & $ M = 1.5\times 10^4$\\ \hline
BLS &  $t=2.7s,$ & $t=149.1s,$ & $t = 970.8 s, $ \\
 &  $M=10^4$ & $M=1.5\times 10^4$ & $M = 1.5\times 10^4$ \\ \hline
GL &  $t=161.4s$ & - & - \\ 
 &  $M=1.8\times 10^5$ & - & - \\ \hline
\end{tabular}  
 \label{tab:timeing}  \caption {Timing for MCMC methods. We report both the number of samples $M$ and the running time of the algorithm $t$. The time for GL on MNIST and Hyperspectral is omitted due to running time being too slow. $J$ denotes the number of fidelity points used. For the voting records, we set $\gamma = 0.2, \beta = 0.4$ for probit and BLS, and $\gamma = 1$, $\beta = 0.1$ for Ginzburg-Landau. For MNIST, we set $\gamma = 0.1, \beta = 0.4$. For Hyperspctral, we set $\gamma = 1.0$, and $\beta = 0.1$. }
\end{table} 

We note that the number of iterations needed for the Ginzburg-Landau model is much higher compared to probit and the Bayesian level-set (BLS) method;
this is caused by the presence of multiple local minima in Ginzburg-Landau,
in contrast to the log concavity of probit. 
probit is slower than BLS due to the fact that evaluations of the cdf function for Gaussians is slow. 

\section{Conclusions and Future Directions}
\label{sec:6}

We introduce a Bayesian approach to uncertainty 
quantification for graph-based classification methods.
We develop algorithms to sample the posterior and to compute
MAP estimators and, through numerical experiments on a suite
of applications, we investigate the properties of the different
Bayesian models, and the algorithms used to study them.

Some future directions of this work include improvement of the current inference method, connections between the different models in this paper, and generalization to multiclass classification,  
for example by vectorizing the latent variable (as in existing non-Bayesian multiclass methods \cite{garcia2014multiclass, merkurjev2013mbo}), and applying multi-dimensional analogues of the likelihood functions used in this paper. Hierarchical methods could also be applied to account for the uncertainty in the various hyperparameters such as the label noise $\gamma$, or the length scale $\epsilon$ in the Ginzburg-Landau model. 
Finally, we could study in more detail the effects of either the spectral projection or the approximation method, either analytically on some tractable toy examples, or empirically on a suite of representative problems.

Studying the modelling assumptions themselves, guided by data, provides 
a  research direction of long term value. Such questions have not been
much studied to  the best of our knowledge. For example the choice of the signum function to relate the latent variable to the categorial data could be questioned, and other models employed; or the level value of $0$ chosen in the level set
approach could be chosen differently, or as a hyper-parameter. Furthermore the form of the prior on the latent variable $u$ could be questioned. We use a Gaussian prior which encodes first and second order statistical information about the unlabelled data. This Gaussian could contain hyper-parameters, of Whittle-Matern type, which could be learnt from the data; and more generally other non-Gaussian priors could and should be 
considered. For instance, in image data, it is often useful to model feature vectors as lying on submanifolds embedded in a higher dimensional space; such structure could be exploited. More generally, addressing the question of which generative models are appropriate for which types of data is an interesting and potentially fruitful research direction.


\vspace{0.2in}

\noindent{\bf Acknowledgements} AMS is grateful to Omiros Papaspiliopoulos
for illuminating discussions about the probit model.

\section*{Appendix}
\addcontentsline{toc}{section}{Appendix}
\renewcommand{\thesubsection}{\Alph{subsection}}

\subsection{Spectral Properties of $L$}
The spectral properties of $L$ are relevant to the spectral 
projection and approximation algorithms from the previous 
section.  Figure \ref{fig:spectra-all} shows the spectra for our
four examples. Note that in all cases the spectrum is contained
in the interval $[0,2]$, consistent with the theoretical result
in \cite[Lemma 1.7, Chapter 1]{chung1997spectral}. The size of
the eigenvalues near to $0$ will determine the accuracy of the spectral
projection algorithm. The rate at which the spectrum accumulates
at a value near  $1$, 
an accumulation which happens for all but the MNIST data set in
our four examples, affects the
accuracy of the spectral approximation algorithm.
There is theory that goes some way towards justifying the observed 
accumulation; 
see \cite[Proposition 9, item 4]{von2008consistency}. This theory
works under the assumption that the features $x_j$ are i.i.d samples from some fixed distribution, and the graph Laplacian is constructed from weights 
$w_{ij} = k(x_i, x_j)$, and $k$ satisfies symmetry, continuity and uniform 
positivity. As a consequence the theory does not apply to  
the graph construction used for the 
MNIST dataset since the $K$-nearest neighbor graph is local; empirically
we find that this results in a graph violating the positivity 
assumption on the weights. This explains why the MNIST example 
does not have a spectrum which accumulates at a value near $1$. 
In the case where the
spectrum does accumulate at a value near $1$, 
the rate can be controlled by adjusting
the parameter $\tau$ appearing in the weight calculations; in the limit
$\tau=\infty$ the graph becomes an unweighted complete graph and its spectrum
comprises the the two points $\{0,\lambda\}$  where $\lambda \to 1$ as
$n \to \infty$ -- see Lemma 1.7 in Chapter 1 of \cite{chung1997spectral}. 

\subsection{MAP Estimation as Semi-supervised Classification Method}
We first prove the convexity of the probit negative log likelihood. 
\begin{proposition}
\label{prop:convex-probit}
Let $\Jp(u)$ be the MAP estimation function for the probit model:
$$\Jp(u)=\frac12 \langle u, Pu \rangle-\sum_{j \in Z'}
{\rm log} \Bigl(\Psi(y(j)u(j);\gamma)\Bigr).$$ If $y(j) \in \{\pm 1\}$ for
all $j$ then $\Jp$ is a convex function in the variable $u$. 
\end{proposition}
\begin{proof}
Since $P$ is semi positive definite, it suffices to show 
that 
$$\sum_{j \in Z'} {\rm log} \Bigl(\Psi(y(j)u(j);\gamma)\Bigr)$$ 
is convex.
Thus, since $y(j)  \in \{\pm 1\}$ for all $j,$ it suffices to show that 
${\rm log} \Bigl(\Psi(x;\gamma)\Bigr)$ is concave with respect to $x$.
Since 
$$\Psi(x; \gamma) = \frac{1}{\sqrt{2\pi\gamma}}\int_{-\infty}^{x}\exp(\frac{-t^2}{2\gamma^2}) dt,$$ we have $\Psi(\gamma x; \gamma) = \Psi(x; 1)$. Since 
scaling $x$ by a constant doesn't change convexity, it suffices to consider
the case $\gamma = 1$. Taking the second derivative with $\gamma = 1$, we see
that it suffices to prove that, for all $x \in \bbR$ and all $\gamma>0$,
\begin{equation}
\Psi^{''}(x; 1)\Psi(x; 1) - \Psi^{'}(x; 1)\Psi^{'}(x; 1)<0. 
\end{equation}
Plugging in the definition of $\Psi$, we have
\begin{equation}
\label{eq:convproof1}
\Psi^{''}(x; 1)\Psi(x; 1) - \Psi^{'}(x; 1)\Psi^{'}(x; 1) =
\frac{-1}{2\pi}\exp(\frac{-x^2}{2})
\Bigl(x \int_{-\infty}^{x}\exp(\frac{-t^2}{2}) dt + \exp(\frac{-x^2}{2})\Bigr).
\end{equation}
Clearly the expression in equation (\ref{eq:convproof1}) is less than $0$ for 
$x\geq 0$. For the case $x<0$, divide equation (\ref{eq:convproof1}) by 
$\frac{1}{2\pi}\exp(\frac{-x^2}{2})$ and note that this gives 
\begin{align}
\begin{aligned}
-x \int_{-\infty}^{x}\exp(\frac{-t^2}{2}) dt - \exp(\frac{-x^2}{2})
&= -x \int_{-\infty}^{x}\exp(\frac{-t^2}{2}) dt+
\int_{-\infty}^{x}t\exp(\frac{-t^2}{2}) dt  \\
&= \int_{-\infty}^{x}(t-x)\exp(\frac{-t^2}{2}) dt  < 0
\end{aligned}
\end{align}
and the proof is complete.
\end{proof}
The probit  MAP estimator thus has a considerable computational advantage over the Ginzburg-Landau MAP estimator, because the latter is not convex and, indeed, can have large numbers
of minimizers. We now discuss numerical results designed to probe the consequences of
convexity, or lack of it, for classification accuracy. The purpose of these experiments is not to match state-of-art results for classification, but rather to study properties of the MAP estimator when varying the feature noise and the percentage of labelled data.

We employ
the two moons  and the MNIST $(4,9)$ data sets. The methods are 
evaluated on a range of values for the percentage of labelled data points, 
and also for a range of values of the feature variance $\sigma$ in the
two moons dataset. The experiments are conducted for $100$ trials with 
different initializations (both two moons and  MNIST $(4,9)$) and 
different data realizations (for two moons only). In Figure \ref{fig:twomoonscomp}, we plot the median classification  
accuracy with error bars from the $100$ trials against the feature variance $\sigma$ for the 
two moons dataset. As well as Ginzburg-Landau and probit classification,
we also display results from spectral clustering based on thresholding
the Feidler eigenvector.  The percentage of fidelity points used is $0.5\%$, $1\%$, 
and $3\%$ for each column. We do the same in  Figure \ref{fig:mnistalgocomp} for  the 4 -9 MNIST data set against the same 
percentages of labelled points.

\begin{figure}[!ht]
\label{fig:twomoonscomp}
\begin{tabular}{ccc}
\subfloat[Fidelity = 0.5$\%$]{\includegraphics[width=40mm]{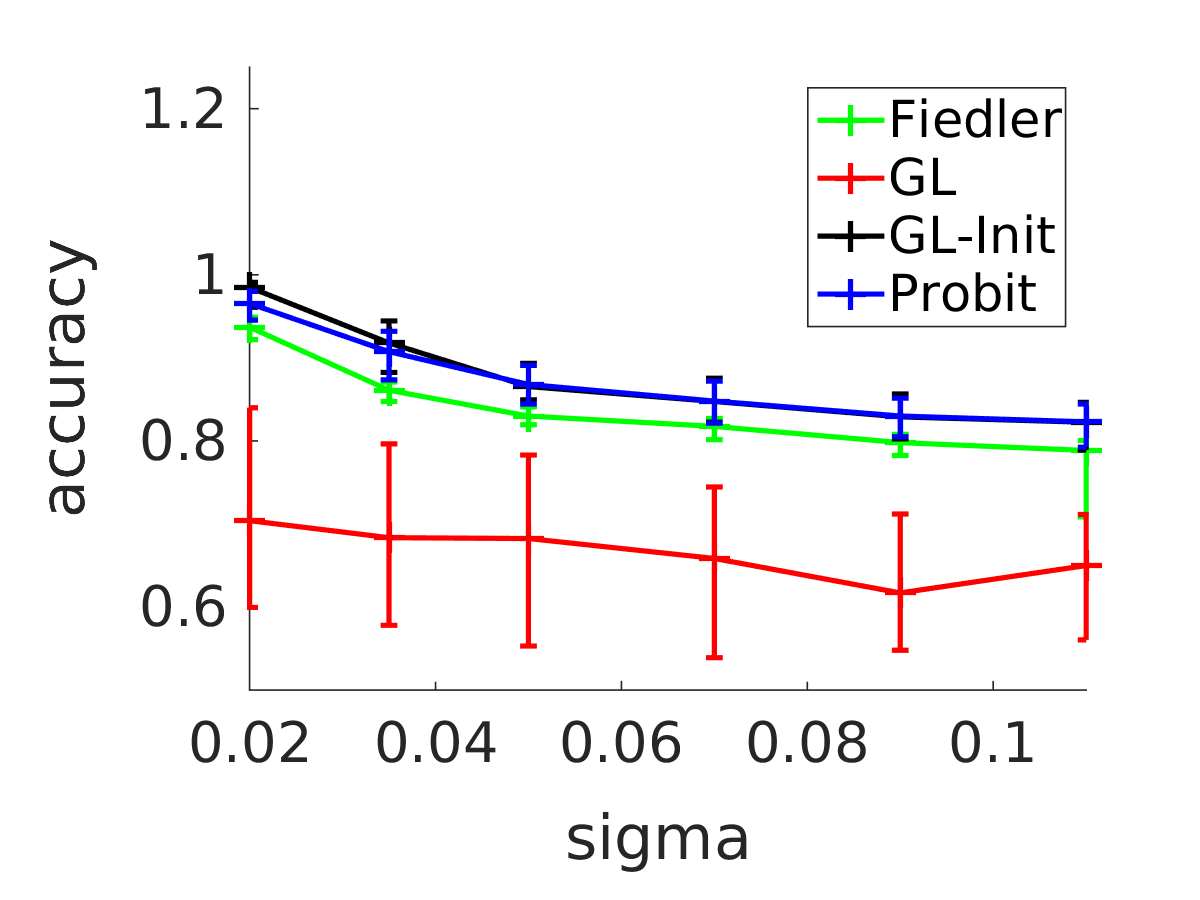}} 
& \subfloat[Fidelity = 1$\%$]{\includegraphics[width=40mm]{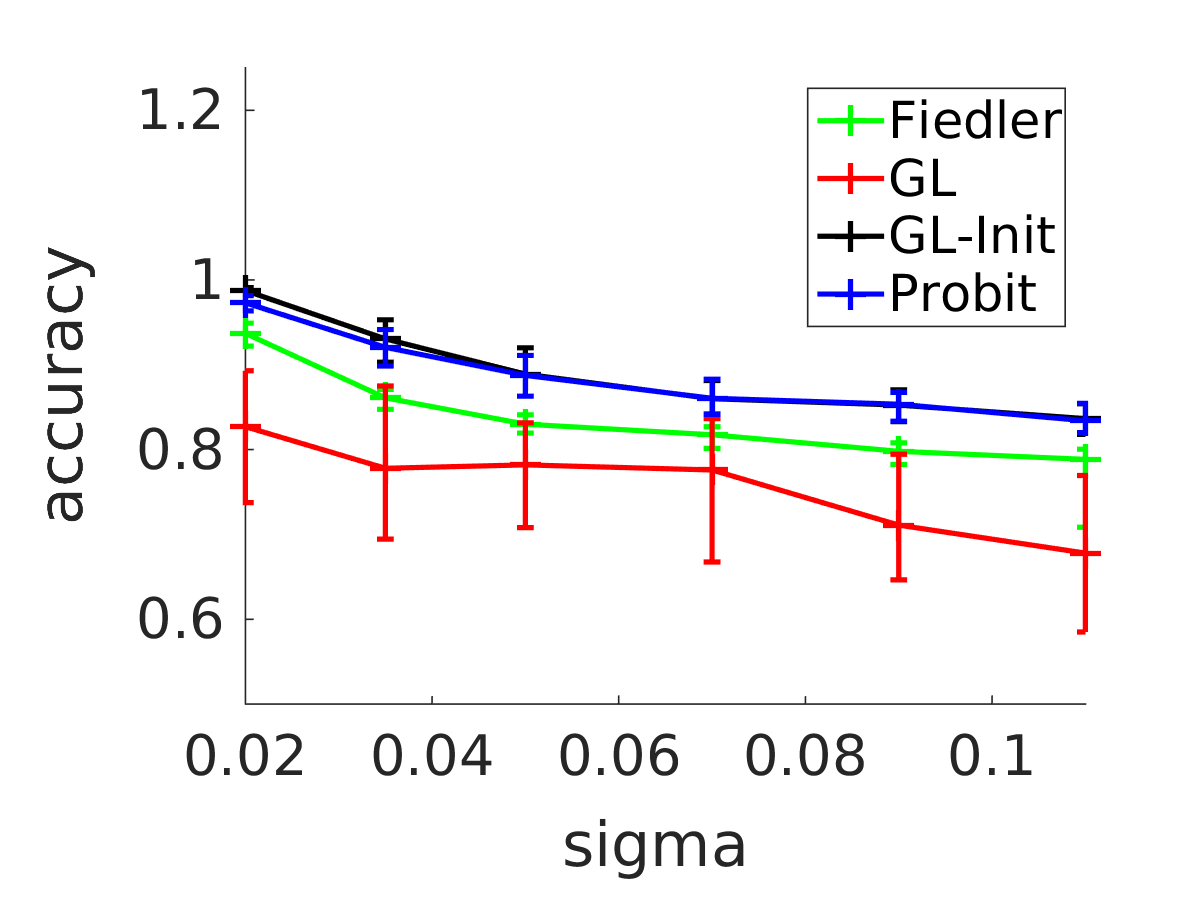}}
& \subfloat[Fidelity = 3$\%$]{\includegraphics[width=40mm]{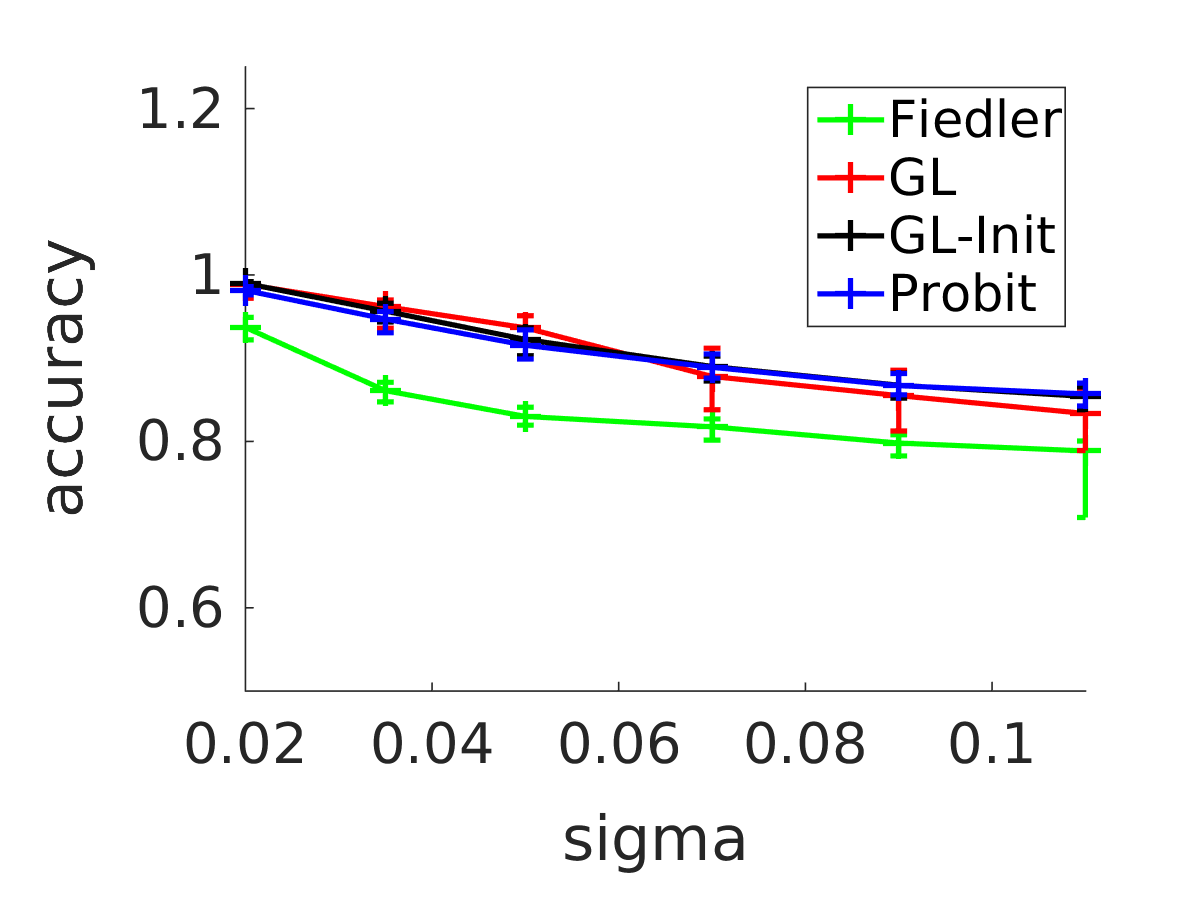}}
\end{tabular}
\caption{Classification accuracy of different algorithms for Two Moons Dataset compared with $\sigma$ and percentage of labelled nodes, with $N = 2,000$. The algorithms used are: Ginzburg-Landau MAP estimator with random initialization, Ginzburg-Landau with initialization given by probit model, probit MAP estimation, and spectral clustering (thresholding the Fiedler vector). For each trial, we generate a realization of the two moons dataset with given $\sigma$ and select randomly a certain percentage of nodes as fidelity, and a total of $50$ trials are run for each combination of parameters. We use spectral projection with number of eigenvectors $Neig = 150$. We plot the median accuracy along with error bars indicating the 25 and 75-th quantile of the classification accuracy of each method.  We set  $\gamma = 0.1$ for the probit model, and $\gamma = 1.0$, $\epsilon = 1.0$ for Ginzburg-Landau. }
\end{figure}

\begin{figure}[!ht]
\centering
\includegraphics[width=75mm]{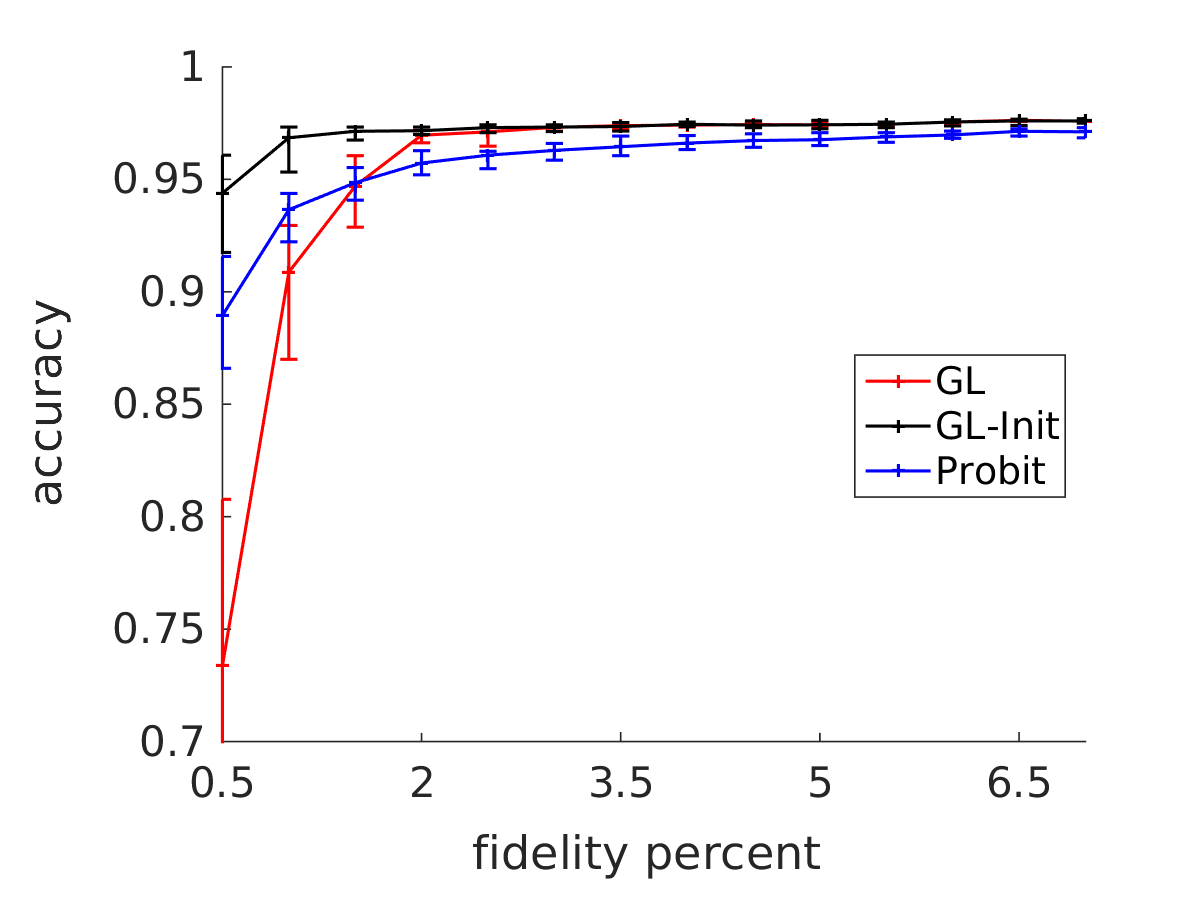}
\caption{Classification accuracy of different algorithms for the 4-9 MNIST dataset versus percentage of labelled nodes. The algorithms used are: Ginzburg-Landau with random initialization, Ginzburg-Landau with initialization given by probit model, probit MAP estimation. For each trial, we select randomly a certain percentage of nodes as fidelity, and a total of $50$ trials are run.  We use spectral projection with number of eigenvectors $Neig = 300$.  We plot the median accuracy along with error bars indicating the 25 and 75-th quantile of the classification accuracy of each method.  We set  $\gamma = 0.1$ for the probit model, and $\gamma = 1.0$, $\epsilon = 1.0$ for Ginzburg-Landau.}\label{fig:mnistalgocomp}
\end{figure}

The non-convexity of the Ginzburg-Landau model can result in large variance 
in classification accuracy; the extent of this depends on the
percentage of observed labels. The existence of sub-optimal local extrema
causes the large variance.
If initialized without information about the classification,
Ginzburg-Landau can perform very badly in comparison with probit. On the
other hand we find that the best performance of the Ginzburg-Landau model,
when initialized at the probit minimizer, is typically slightly
better than the probit model. 

We note that the probit model is convex and theoretically should have results independent of the initialization. However, we see there are still small variations in the classification result from different initializations. This is due to slow convergence of gradient methods caused by the flat-bottomed well of the probit log-likelihood.  As mentioned above this can be understood by noting
that, for small gamma, probit and level-set are closely related and that
the level-set MAP estimator does not exist -- minimizing sequences
converge to zero, but the infimum is not attained at zero.


\begin{thebibliography}{10}

\bibitem{alaiz2016convex}
{\sc C.~M. Ala{\'\i}z, M.~Fanuel, and J.~A. Suykens}, {\em Convex formulation
  for kernel {PCA} and its use in semi-supervised learning}, arXiv preprint
  arXiv:1610.06811,  (2016).

\bibitem{alaiz2016robust}
{\sc C.~M. Ala{\'\i}z, M.~Fanuel, and J.~A. Suykens}, {\em Robust
  classification of graph-based data}, arXiv preprint arXiv:1612.07141,
  (2016).

\bibitem{belkin2004regularization}
{\sc M.~Belkin, I.~Matveeva, and P.~Niyogi}, {\em Regularization and
  semi-supervised learning on large graphs}, in International Conference on
  Computational Learning Theory, Springer, 2004, pp.~624--638.

\bibitem{belkin2006manifold}
{\sc M.~Belkin, P.~Niyogi, and V.~Sindhwani}, {\em Manifold regularization: A
  geometric framework for learning from labeled and unlabeled examples},
  Journal of Machine Learning Research, 7 (2006), pp.~2399--2434.

\bibitem{berthod1996bayesian}
{\sc M.~Berthod, Z.~Kato, S.~Yu, and J.~Zerubia}, {\em Bayesian image
  classification using {Markov} random fields}, Image and Vision Computing, 14
  (1996), pp.~285--295.

\bibitem{luos}
{\sc A.~Bertozzi, X.~Luo, and A.~M. Stuart}, {\em {Scalable sampling methods
  for graph-based semi-supervised learning}}, To be submitted. arXiv preprint
  arXiv:,  (2017).

\bibitem{bertozzi2012diffuse}
{\sc A.~L. Bertozzi and A.~Flenner}, {\em Diffuse interface models on graphs
  for classification of high dimensional data}, Multiscale Modeling \&
  Simulation, 10 (2012), pp.~1090--1118.

\bibitem{beskos2009optimal}
{\sc A.~Beskos, G.~Roberts, and A.~Stuart}, {\em Optimal scalings for local
  {Metropolis-Hastings} chains on nonproduct targets in high dimensions}, The
  Annals of Applied Probability,  (2009), pp.~863--898.

\bibitem{BRSV08}
{\sc A.~Beskos, G.~Roberts, A.~M. Stuart, and J.~Voss}, {\em \uppercase{MCMC}
  methods for diffusion bridges.}, Stochastics and Dynamics, 8 (2008),
  pp.~319--350.

\bibitem{bishop2007pattern}
{\sc C.~Bishop}, {\em Pattern recognition and machine learning (information
  science and statistics), 1st edn. 2006. corr. 2nd printing edn}, Springer,
  New York,  (2007).

\bibitem{blum2001learning}
{\sc A.~Blum and S.~Chawla}, {\em Learning from labeled and unlabeled data
  using graph mincuts},  (2001).

\bibitem{boykov1998markov}
{\sc Y.~Boykov, O.~Veksler, and R.~Zabih}, {\em {Markov} random fields with
  efficient approximations}, in Computer vision and pattern recognition, 1998.
  Proceedings. 1998 IEEE computer society conference on, IEEE, 1998,
  pp.~648--655.

\bibitem{boykov2001fast}
{\sc Y.~Boykov, O.~Veksler, and R.~Zabih}, {\em Fast approximate energy
  minimization via graph cuts}, IEEE Transactions on pattern analysis and
  machine intelligence, 23 (2001), pp.~1222--1239.

\bibitem{boykov2001interactive}
{\sc Y.~Y. Boykov and M.-P. Jolly}, {\em Interactive graph cuts for optimal
  boundary \& region segmentation of objects in nd images}, in Computer Vision,
  2001. ICCV 2001. Proceedings. Eighth IEEE International Conference on,
  vol.~1, IEEE, 2001, pp.~105--112.

\bibitem{broadwater2011primer}
{\sc J.~B. Broadwater, D.~Limsui, and A.~K. Carr}, {\em A primer for chemical
  plume detection using {LWIR} sensors}, Technical Paper, National Security
  Technology Department, Las Vegas, NV,  (2011).

\bibitem{buhler2009spectral}
{\sc T.~B{\"u}hler and M.~Hein}, {\em Spectral clustering based on the graph
  p-{Laplacian}}, in Proceedings of the 26th Annual International Conference on
  Machine Learning, ACM, 2009, pp.~81--88.

\bibitem{chung1997spectral}
{\sc F.~R. Chung}, {\em Spectral graph theory}, vol.~92, American Mathematical
  Soc., 1997.

\bibitem{CRSW08}
{\sc S.~L. Cotter, G.~O. Roberts, A.~M. Stuart, and D.~White}, {\em
  \uppercase{MCMC} methods for functions: modifying old algorithms to make them
  faster.}, Statistical Science, 28 (2013), pp.~424--446.

\bibitem{dahlhaus1992complexity}
{\sc E.~Dahlhaus, D.~S. Johnson, C.~H. Papadimitriou, P.~D. Seymour, and
  M.~Yannakakis}, {\em The complexity of multiway cuts}, in Proceedings of the
  twenty-fourth annual ACM symposium on theory of computing, ACM, 1992,
  pp.~241--251.

\bibitem{dunlop2016hierarchical}
{\sc M.~M. Dunlop, M.~A. Iglesias, and A.~M. Stuart}, {\em Hierarchical
  bayesian level set inversion}, arXiv preprint arXiv:1601.03605,  (2016).

\bibitem{fowlkes2004spectral}
{\sc C.~Fowlkes, S.~Belongie, F.~Chung, and J.~Malik}, {\em Spectral grouping
  using the {Nystr\"om} method}, IEEE transactions on pattern analysis and
  machine intelligence, 26 (2004), pp.~214--225.

\bibitem{garcia2014multiclass}
{\sc C.~Garcia-Cardona, E.~Merkurjev, A.~L. Bertozzi, A.~Flenner, and A.~G.
  Percus}, {\em Multiclass data segmentation using diffuse interface methods on
  graphs}, IEEE transactions on pattern analysis and machine intelligence, 36
  (2014), pp.~1600--1613.

\bibitem{hammond2011wavelets}
{\sc D.~K. Hammond, P.~Vandergheynst, and R.~Gribonval}, {\em Wavelets on
  graphs via spectral graph theory}, Applied and Computational Harmonic
  Analysis, 30 (2011), pp.~129--150.

\bibitem{hartog2016nonparametric}
{\sc J.~Hartog and H.~van Zanten}, {\em Nonparametric bayesian label prediction
  on a graph}, arXiv preprint arXiv:1612.01930,  (2016).

\bibitem{higham2004unified}
{\sc D.~J. Higham and M.~Kibble}, {\em A unified view of spectral clustering},
  University of Strathclyde mathematics research report, 2 (2004).

\bibitem{hu2015multi}
{\sc H.~Hu, J.~Sunu, and A.~L. Bertozzi}, {\em Multi-class graph {Mumford-Shah}
  model for plume detection using the {MBO} scheme}, in Energy Minimization
  Methods in Computer Vision and Pattern Recognition, Springer, 2015,
  pp.~209--222.

\bibitem{iglesias2015bayesian}
{\sc M.~A. Iglesias, Y.~Lu, and A.~M. Stuart}, {\em {A Bayesian Level Set
  Method for Geometric Inverse Problems}}, Interfaces and Free Boundary
  Problems, arXiv preprint arXiv:1504.00313,  (2015).

\bibitem{kapoor2005hyperparameter}
{\sc A.~Kapoor, Y.~Qi, H.~Ahn, and R.~Picard}, {\em Hyperparameter and kernel
  learning for graph based semi-supervised classification}, in NIPS, 2005,
  pp.~627--634.

\bibitem{lecun1998mnist}
{\sc Y.~LeCun, C.~Cortes, and C.~J. Burges}, {\em The {MNIST} database of
  handwritten digits, online at http://yann.lecun. com/exdb/mnist/}, 1998.

\bibitem{li2012markov}
{\sc S.~Z. Li}, {\em {Markov} random field modeling in computer vision},
  Springer Science \& Business Media, 2012.

\bibitem{lu2011multi}
{\sc S.~Lu and S.~V. Pereverzev}, {\em Multi-parameter regularization and its
  numerical realization}, Numerische Mathematik, 118 (2011), pp.~1--31.

\bibitem{madry2010fast}
{\sc A.~Madry}, {\em Fast approximation algorithms for cut-based problems in
  undirected graphs}, in Foundations of Computer Science (FOCS), 2010 51st
  Annual IEEE Symposium on, IEEE, 2010, pp.~245--254.

\bibitem{merkurjev2013mbo}
{\sc E.~Merkurjev, T.~Kostic, and A.~L. Bertozzi}, {\em An {MBO} scheme on
  graphs for classification and image processing}, SIAM Journal on Imaging
  Sciences, 6 (2013), pp.~1903--1930.

\bibitem{merkurjev2014graph}
{\sc E.~Merkurjev, J.~Sunu, and A.~L. Bertozzi}, {\em Graph {MBO} method for
  multiclass segmentation of hyperspectral stand-off detection video}, in Image
  Processing (ICIP), 2014 IEEE International Conference on, IEEE, 2014,
  pp.~689--693.

\bibitem{neal}
{\sc R.~Neal}, {\em Regression and classification using {Gaussian} process
  priors}, Bayesian Statistics, 6, p.~475.
\newblock Available at http://www.cs.toronto.
  edu/~radford/valencia.abstract.html.

\bibitem{osting2014minimal}
{\sc B.~Osting, C.~D. White, and {\'E}.~Oudet}, {\em Minimal {Dirichlet} energy
  partitions for graphs}, SIAM Journal on Scientific Computing, 36 (2014),
  pp.~A1635--A1651.

\bibitem{owhadi2}
{\sc H.~Owhadi, C.~Scovel, and T.~Sullivan}, {\em On the brittleness of
  {Bayesian} inference}, SIAM Review, 57 (2015), pp.~566--582.

\bibitem{owhadi}
{\sc H.~Owhadi, C.~Scovel, T.~J. Sullivan, M.~McKerns, and M.~Ortiz}, {\em
  Optimal uncertainty quantification}, SIAM Review, 55 (2013), pp.~271--345.

\bibitem{R}
{\sc G.~O. Roberts, A.~Gelman, W.~R. Gilks, et~al.}, {\em Weak convergence and
  optimal scaling of random walk {Metropolis} algorithms}, The Annals of
  Applied Probability, 7 (1997), pp.~110--120.

\bibitem{shuman2011semi}
{\sc D.~I. Shuman, M.~Faraji, and P.~Vandergheynst}, {\em Semi-supervised
  learning with spectral graph wavelets}, in Proceedings of the International
  Conference on Sampling Theory and Applications (SampTA),
  no.~EPFL-CONF-164765, 2011.

\bibitem{sindhwani200611}
{\sc V.~Sindhwani, M.~Belkin, and P.~Niyogi}, {\em The geometric basis of
  semi-supervised learning},  (2006).

\bibitem{smith}
{\sc R.~C. Smith}, {\em Uncertainty quantification: theory, implementation, and
  applications}, vol.~12, SIAM, 2013.

\bibitem{subramanya2011semi}
{\sc A.~Subramanya and J.~Bilmes}, {\em Semi-supervised learning with measure
  propagation}, Journal of Machine Learning Research, 12 (2011),
  pp.~3311--3370.

\bibitem{sull}
{\sc T.~J. Sullivan}, {\em Introduction to uncertainty quantification},
  vol.~63, Springer, 2015.

\bibitem{talukdar2009new}
{\sc P.~Talukdar and K.~Crammer}, {\em New regularized algorithms for
  transductive learning}, Machine Learning and Knowledge Discovery in
  Databases,  (2009), pp.~442--457.

\bibitem{van2012gamma}
{\sc Y.~Van~Gennip and A.~L. Bertozzi}, {\em {$\Gamma $}-convergence of graph
  {Ginzburg-Landau} functionals}, Advances in Differential Equations, 17
  (2012), pp.~1115--1180.

\bibitem{von2007tutorial}
{\sc U.~Von~Luxburg}, {\em A tutorial on spectral clustering}, Statistics and
  Computing, 17 (2007), pp.~395--416.
  
  \bibitem{von2008consistency}
{\sc U.~Von~Luxburg, M.~Belkin, and O.~Bousquet}, {\em Consistency of spectral
  clustering}, The Annals of Statistics,  (2008), pp.~555--586.

\bibitem{wahba1990spline}
{\sc G.~Wahba}, {\em Spline models for observational data}, SIAM, 1990.

\bibitem{williams1996gaussian}
{\sc C.~K. Williams and C.~E. Rasmussen}, {\em {Gaussian Processes for
  Regression}},  (1996).

\bibitem{williams2000using}
{\sc C.~K. Williams and M.~Seeger}, {\em Using the {Nystr{\"o}m} method to
  speed up kernel machines}, in Proceedings of the 13th International
  Conference on Neural Information Processing Systems, MIT press, 2000,
  pp.~661--667.

\bibitem{xiu}
{\sc D.~Xiu}, {\em Numerical Methods For Stochastic Computations: A Spectral
  Method Approach}, Princeton University Press, 2010.

\bibitem{zelnik2004self}
{\sc L.~Zelnik-Manor and P.~Perona}, {\em Self-tuning spectral clustering}, in
  Advances in neural information processing systems, 2004, pp.~1601--1608.

\bibitem{zhou2004learning}
{\sc D.~Zhou, O.~Bousquet, T.~N. Lal, J.~Weston, and B.~Sch{\"o}lkopf}, {\em
  Learning with local and global consistency}, Advances in neural information
  processing systems, 16 (2004), pp.~321--328.

\bibitem{zhu2005semi}
{\sc X.~Zhu}, {\em Semi-supervised learning literature survey}, Technical
  Report TR1530.

\bibitem{zhu2003semi}
{\sc X.~Zhu, Z.~Ghahramani, J.~Lafferty, et~al.}, {\em Semi-supervised learning
  using {Gaussian} fields and harmonic functions}, in ICML, vol.~3, 2003,
  pp.~912--919.

\bibitem{zhu2003semib}
{\sc X.~Zhu, J.~D. Lafferty, and Z.~Ghahramani}, {\em Semi-supervised learning:
  From {Gaussian} fields to {Gaussian} processes},  (2003).

\end{thebibliography}
\end{document}